\documentclass[10pt,a4paper]{article}
\usepackage[english]{babel}
\usepackage[cp850]{inputenc}
\usepackage[dvips]{graphicx}
\usepackage{amsmath,amsfonts,amsthm,amssymb}
\usepackage{mathrsfs}
\usepackage[usenames, dvipsnames]{color}
\usepackage{xspace}
\usepackage{adjustbox}
\usepackage[colorlinks=true,citecolor=blue,linkcolor=blue,urlcolor=blue,pdfstartview=FitH]{hyperref}
\usepackage{nameref}
\usepackage{enumitem}
\usepackage{tabularx, booktabs}
\usepackage{footnote}
\makesavenoteenv{tabular}
\usepackage{natbib}

\usepackage{algorithm}
\usepackage{algorithmicx}
\usepackage{algpseudocode}
\usepackage{caption}
\usepackage{geometry}
\usepackage{tikz,tkz-tab}
\usetikzlibrary{fit,calc,positioning,trees}
\usetikzlibrary{arrows}
\tikzstyle{level 1}=[level distance=4cm, sibling distance=2.5cm]
\tikzstyle{level 2}=[level distance=8cm, sibling distance=0.6cm]
\tikzstyle{bag} = [text width=4em, text centered]
\tikzstyle{end} = [circle, minimum width=3pt,fill, inner sep=0pt]
\usetikzlibrary{decorations.pathmorphing}
\usetikzlibrary{decorations.pathreplacing}
\usetikzlibrary{decorations.shapes}
\usetikzlibrary{decorations.text}
\usetikzlibrary{decorations.markings}
\usetikzlibrary{decorations.fractals}
\usetikzlibrary{decorations.footprints}

\graphicspath{{image/}}

\DeclareGraphicsExtensions{.pdf,.png}

\makeatletter
\newcommand{\manuallabel}[2]{\def\@currentlabel{#2}\label{#1}}
\makeatother
\newcommand\numberthis{\addtocounter{equation}{1}\tag{\theequation}}

\algnewcommand{\Inputs}[1]{%
  \State \textbf{Inputs:}
  \Statex \hspace*{\algorithmicindent}\parbox[t]{.8\linewidth}{\raggedright #1}
}
\algnewcommand{\Initialize}[1]{%
  \State \textbf{Initialize:}
  \Statex \hspace*{\algorithmicindent}\parbox[t]{.8\linewidth}{\raggedright #1}
}

\makeatletter
\newcommand*{\inlineequation}[2][]{%
  \begingroup
    \refstepcounter{equation}%
    \ifx\\#1\\%
    \else
      \label{#1}%
    \fi
    \relpenalty=10000 %
    \binoppenalty=10000 %
    \ensuremath{%
      #2 %
    }%
    ~\@eqnnum
  \endgroup
}
\makeatother

\makeatletter
\newtheorem*{rep@theorem}{\rep@title}
\newcommand{\newreptheorem}[2]{%
\newenvironment{rep#1}[1]{%
 \def\rep@title{#2 \ref{##1}}%
 \begin{rep@theorem}}%
 {\end{rep@theorem}}}
\makeatother

\usepackage{hyperref}
\usepackage{amsmath}
\usepackage{scrextend}

\DeclareMathAlphabet{\mathpzc}{OT1}{pzc}{m}{it}
\DeclareMathOperator{\Tr}{Tr}

\def\Esp{\mathbb{E}}

\def\Tr{{\text{Tr}}}

\def\é{\'{e}}
\def\è{\`{e}}
\def\ê{\^{e}}
\def\à{\`{a}}
\def\ô{\^{o}}

\newcommand{\Var}{\mathsf{Var}}
\newcolumntype{C}[1]{>{\centering\arraybackslash}p{#1}}
\newcolumntype{L}[1]{>{\raggedleft\arraybackslash}p{#1}}
\newcolumntype{R}[1]{>{\raggedright\arraybackslash}p{#1}}

\newtheorem{theo}{Theorem}
\newtheorem{prop}{Proposition}
\newtheorem{lem}{Lemma}
\newtheorem{Assumption}{Assumption}

\newreptheorem{theorem}{Theorem}
\newreptheorem{lemma}{Lemma}
\newreptheorem{prop}{Proposition}

\newcommand{\bp}{\mathbb{P}}

\title{GANs Training: A Game and Stochastic Control Approach}
\author{Xin Guo\thanks{Department of Industrial Engineering and Operations Research
University of California, Berkeley, Berkeley, CA 94704} \ \  \ Othmane Mounjid\footnotemark[1]}
\date{\today\\}
\begin{document}
\maketitle

\begin{abstract}
Training generative adversarial networks (GANs) is known to be difficult, especially for financial time series. This paper first analyzes the well-posedness problem in GANs minimax games and the convexity issue in GANs objective functions. It then
proposes a stochastic control framework for hyper-parameters tuning in GANs training.  The weak form of dynamic programming principle and the uniqueness and the existence of the value function in the viscosity sense for the corresponding minimax game are established.  In particular,  explicit forms for the optimal adaptive learning rate and batch size are derived and are shown to depend on the convexity of the objective function, revealing a relation between improper choices of learning rate and explosion in GANs training. Finally,  empirical studies demonstrate that training algorithms incorporating this adaptive control approach outperform the standard ADAM method in terms of convergence and robustness.

From GANs training perspective, the analysis in this paper provides analytical support for the popular practice of ``clipping'', and suggests that the convexity and well-posedness issues in GANs may be tackled through appropriate choices of hyper-parameters.
\end{abstract}

\section{Introduction}
\label{sec:Intro}

Generative adversarial networks (GANs) belong to the class of generative models.
The key idea behind GANs \citep{goodfellow2014generative} is to add a discriminator network in order to improve the data generation process. GANs can therefore be viewed as competing games between two neural networks: a generator network and a discriminator network. The generator network attempts to fool the discriminator network by converting random noise into sample data, while the discriminator network tries to identify whether the sample data is fake or real.
This powerful idea behind GANs leads to a versatile class of generative models with a wide range of successful applications from image generation to natural language processing \citep{denton2015deep,radford2015unsupervised, yeh2016semantic, ledig2016others, zhu2016generative, reed2016generative, vondrick2016generating, luc2016semantic, ghosh2016contextual}.\\

Inspired by the success of GANs for computer vision, there is a surge of interest in applying GANs for financial time series data generation. In such a context, the key contributions consist of adapting divergence functions and network architectures in order to cope with the time series nature and the dependence structure of financial data. For example, Quant-GAN \citep{Wiese2019,wiese2020quant} uses a TCN structure,  C-Wasserstein GAN  \citep{li2020generating} uses the Wasserstein distance as \citep{arjovsky2017wasserstein} and \citep{gulrajani2017improved}, FIN-GAN  \citep{takahashi2019modeling} generates synthetic financial data, Corr-GAN
\citep{marti2020corrgan} adapts DCGAN structure of \citep{radford2015unsupervised} for correlation matrix of asset returns, C-GAN \citep{fu2019time} follows \citep{mirza2014conditional} to simulate realistic conditional scenarios,  Sig-Wasserstein-GAN \citep{ni2020conditional} generates orders and transactions, and Tail-GAN \citep{dionelis2020tail} deals with risk management. Furthermore, activities from industrial AI labs such as JP Morgan \citep{storchan2020mas} and American Express \citep{efimov2020using} have further elevated GANs as promising tools for synthetic data generation and for model testing. (See also reviews by \citep{cao2021generative} and \citep{eckerli2021generative} for more details).\\

Despite this empirical success, there are well-recognized issues in GANs training, including the vanishing gradient when there is an imbalance between the generator and the discriminator training \citep{arjovsky2017towards}, and the mode collapse when the generator learns to produce only a specific range of samples \citep{Salimans2016}. In particular, it is hard to reproduce realistic financial time series due to well-documented stylized facts \citep{chakraborti2011econophysics,cont2001empirical} such as
heavy-tailed-and-skewed distribution of asset returns,  volatility clustering, the
leverage effect, and the Zumbach effect  \citep{zumbach2001heterogeneous}. For instance, \citep{takahashi2019modeling} shows that batch normalization tends to yield large fluctuations in the generated time series as well as strong auto-correlation. Moreover, Figure \ref{Fig:Plot_loss_generator_acc_discrim_CNNGan} below highlights the difficulty of convergence for GANs when a convolutional neural network is used to generate financial time series: the discriminator does not reach the desired accuracy level of $50\%$ and the generator exhibits a non-vanishing loss. (See Section \ref{sec:Exp} for a detailed description of data sources).

\begin{figure}[h!]
    \center
	 \hspace{-2.cm} (a) Discriminator accuracy  \hspace{1.5cm} (b) Generator loss  \\
	\includegraphics[width=0.45\textwidth]{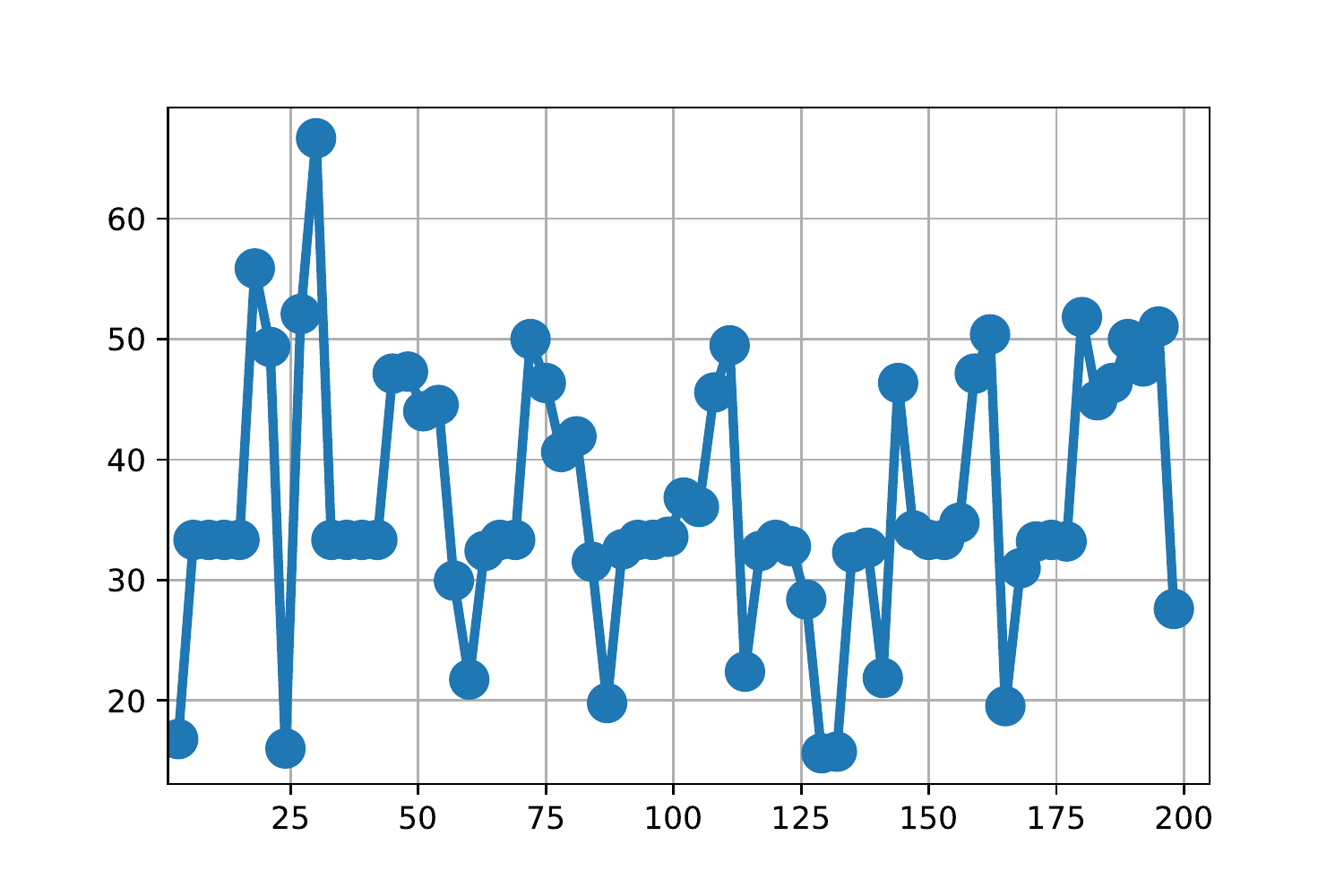}	\includegraphics[width=0.45\textwidth]{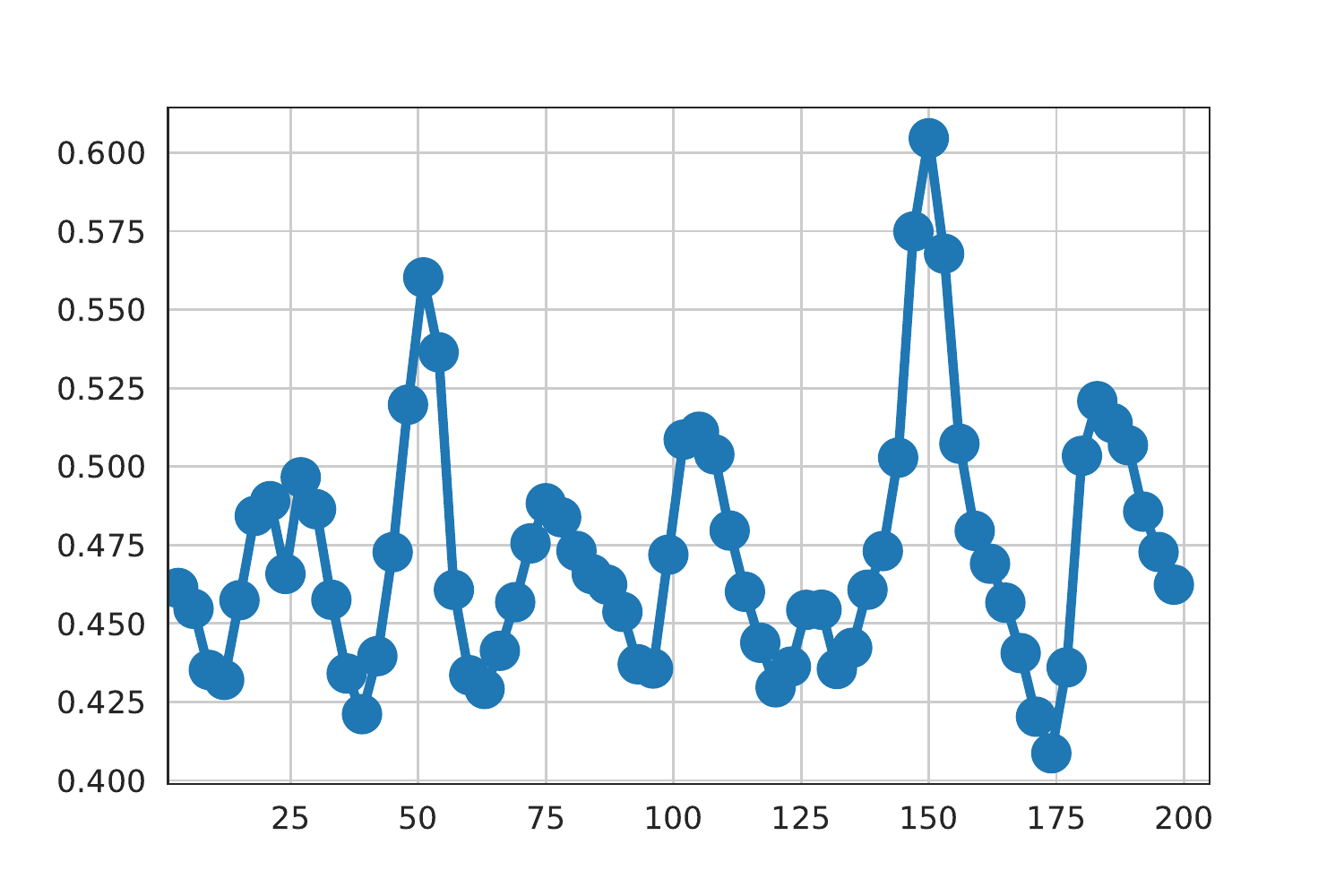}	
    \caption{Discriminator accuracy in (a), and generator loss in (b).}
    \label{Fig:Plot_loss_generator_acc_discrim_CNNGan}
\end{figure}

In response, there have been a number of theoretical studies for GANs training. \citep{Berard2020} proposes a visualization method for the GANs training process through the gradient vector field of loss functions,
 \citep{Mescheder2018} demonstrates that regularization improves  the convergence performance of GANs, \citep{Conforti2020} and \citep{Domingo-Enrich2020} analyze general minimax games including  GANs,
and connect the mixed Nash equilibrium of the game with the invariant measure of Langevin dynamics. In the same spirit, \citep{hsieh2019finding} proposes a sampling algorithm that converges towards mixed Nash equilibria, and then \citep{kamalaruban2020robust} shows how this algorithm escapes saddle points. Recently, \citep{cao2020approximation} establishes
the continuous-time approximation for the discrete-time GANs training by coupled stochastic differential
equations, enabling the convergence analysis of GANs training via stochastic tools.

\paragraph{Our work.} The focus of this paper is to analyze GANs training in the stochastic control and game framework. It starts by revisiting vanilla GANs from the original work of \citep{goodfellow2014generative} and identifies through detailed convexity analysis one of the culprits behind the convergence issue for GANs: the lack of convexity in GANs objective function hence the general well-posedness issue in GANs minimax games. It then reformulates GANs problem as a stochastic game and uses it to study the optimal control of learning rate (and its equivalence to the optimal choice of time scale) and optimal batch size.\\

To facilitate the analysis of this type of minimax games, this paper first establishes the weak form of dynamics programming principle for a general class of stochastic games in the spirit of \citep{bouchard2011weak}; it then focuses on the particular minimax games of GANs training, by analyzing the existence and the uniqueness of viscosity solutions to Issac-Bellman equations. In particular, it obtains an explicit form for the optimal adaptive learning rate and optimal batch size which depend on the convexity of the objective function for the game.
Finally, by experimenting on synthetic data drawn either from the Gaussian distribution or the Student t-distribution and financial data collected from the Quandl Wiki Prices database, it demonstrates that training algorithms incorporating our adaptive control methodology outperform the standard ADAM optimizer, in terms of both convergence and robustness.\\
 
Note that the dynamic programming principle for stochastic games has been proved in a deterministic setting by \citep{evans1984differential}, and recently extended to the stochastic framework for diffusions under boundedness and regularity conditions \citep{krylov2014dynamic,bayraktar2013weak,sirbu2014martingale,sirbu2014stochastic}. The dynamic programming principle established in this paper is without the continuity assumption and is for a more general class of stochastic differential games beyond diffusions. In addition, upper and lower bounds of the value function are obtained.\\

In terms of GANs training, our analysis provides analytical support for the popular practice of ``clipping'' in GANs \citep{arjovsky2017towards}; and suggests that the convexity and well-posedness issues associated with GANs problems may be resolved by appropriate choices of hyper-parameters. In addition, our study presents a precise relation between explosion in GANs training and improper choices of the learning rate. It also reveals an interesting connection between the optimal learning rate and the standard Newton algorithm.

\paragraph{Notations.} Throughout this paper, the following notations are used:
\begin{itemize}
\item For any vector $x \in \mathbb{R}^d$ with $d \in \mathbb{N}^*$, denote by the operator $\nabla_x$ the gradient with respect to the coordinates of $x$. When there is no subscript $x$, the operator $\nabla$ refers to the standard gradient operator.
\item For any $d \in \mathbb{N}^*$, the set $\mathcal{M}_{\mathbb{R}}(d)$ is the space of $d\times d$ matrices with real coefficients.
\item For any vector $m \in \mathbb{R}^d$ and symmetric positive-definite matrix $ A \in \mathcal{M}_{\mathbb{R}}(d)$ with $d \in \mathbb{N}^*$, denote by $ N(m,A)$ the Gaussian distribution with mean $m$ and covariance matrix $A$.
\item  $\Esp_X$ emphasizes on the dependence of expectation with respect to the distribution of $X$. 
\end{itemize}

\section{GANs: Well-posedness and Convexity}
\label{sec:GAN_wellpos}

\paragraph{GANs as generative models.} GANs fall into the category of generative models. The procedure of generative modeling is to approximate an {\it unknown true} distribution ${\mathbb P}_X$ of a random variable $X$ from a sample space $\cal{X}$ by constructing a class of suitable parametrized probability distributions ${\mathbb P}_\theta$. That is, given a latent space $\cal{Z}$, define a latent variable $Z\in\cal{Z}$ with a fixed probability distribution 
and a family of functions $G_\theta: \cal{Z} \to \cal{X}$ parametrized by $\theta$. Then, $\bp_\theta$ can be seen as the probability distribution of $G_\theta(Z)$.\\

To approximate ${\mathbb P}_X$, GANs use two competing neural networks: a generator network for the function $G_\theta$, and a discriminator network $D_w$ parametrized by $w$. The discriminator $D_w$ assigns a score between $0$ and $1$ to each sample. A score closer to $1$ indicates that the sample is more likely to be from the true distribution.
 GANs are trained by optimizing $G_\theta$ and $D_w$ iteratively until $D_w$ can no longer distinguish between true and generated samples and assigns a score close to $0.5$.

\paragraph{Equilibrium of GANs as minimax games.} Under a fixed network architecture, the parametrized GANs optimization problem can be viewed as  the following minimax game:
\begin{align}
 \min_{\theta \in \mathbb{R}^{N}} \max_{w \in \mathbb{R}^{M}} g(w,\theta),
\label{Eq:min_pbm_gan}
\end{align}
with $g:\mathbb{R}^M \times \mathbb{R}^N \rightarrow \mathbb{R}$ the objective function. In vanilla GANs,
\begin{equation}
g(w,\theta) = \Esp_X\big[\log(D_w(X))\big] + \Esp_Z\big[\log\big(1 - D_w(G_\theta(Z))\big)\big],
\label{Eq:vanilla_gan}
\end{equation}
where $D_w: \mathbb{R}^M  \rightarrow \mathbb{R}$ is  the discriminator network, $G_\theta: \mathbb{R}^N  \rightarrow \mathbb{R}$ is the generator network, $X$ represents the unknown data distribution, and $Z$ is the latent variable. \\

From a game theory viewpoint, the objective in \eqref{Eq:min_pbm_gan}, when attained, is in fact the upper value of the two-player zero-sum game of GANs.  If there exists a locally optimal pair of parameters $(\theta^*,w^*)$ for \eqref{Eq:min_pbm_gan}, then $(\theta^*,w^*)$ is a Nash equilibrium, i.e., no player can do better by unilaterally deviating from her strategy. \\

To guarantee the existence of a Nash equilibrium, convexity or concavity conditions are required at least locally. From an optimization perspective, these convexity/concavity conditions also ensure the absence of a duality gap, as suggested by Sion's generalized minimax theorem in \citep{sion1958general} and \citep{von1959theory}.

\subsection{GANs Training and Convexity}
\manuallabel{sec:GANsTrainConv}{2.2}

GANs are trained by stochastic gradient algorithms (SGA). In GANs training, if $g$ is convex (resp. concave) in $\theta$ (resp. $\omega$), then it is possible to decrease (resp. increase) the objective function $g$ by moving in the opposite (resp. same) direction of the gradient. \\

It is well known that SGA may not converge without suitable convexity properties on $g$, even when the Nash equilibrium is unique for a minimax game. For instance,  $g(x,y)=xy$ clearly admits the point $(0,0)$ as a unique Nash equilibrium. However, as illustrated in Figure \ref{Fig:spiral},  SGA fails to converge since $g$ is neither strictly concave nor strictly convex in $x$ or $y$. \\

\begin{figure}[h!]
    \center
    \includegraphics[width=0.45\textwidth]{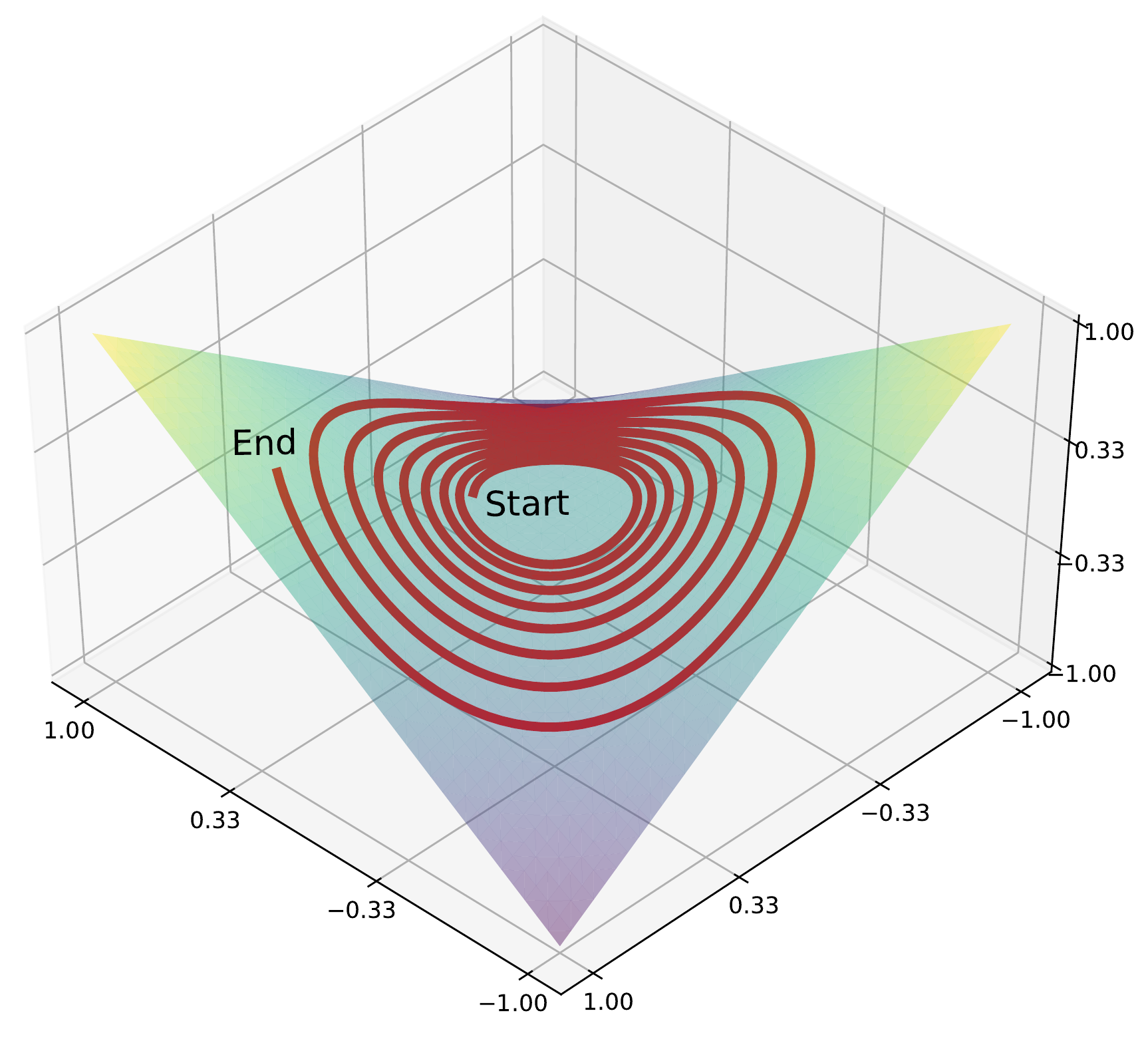}
    \caption{Plot of the parameters values when using SGA to solve the minimax problem $\min_{y \in \mathbb{R}} \max_{x \in \mathbb{R}} x y$. The label ``Start" (resp. ``End") indicates the initial (resp. final) value of the parameters.}
    \label{Fig:spiral}
\end{figure}

Moreover, convexity properties are easy to violate with the composition of multiple layers in a neural network even for a simple GANs model, as illustrated in the following example. 

\paragraph{Counterexample.} 
Take the  vanilla GANs with $g$  in \eqref{Eq:vanilla_gan}. Take two normal distributions for $X$ and $Z$ such that $X \sim N(m,\sigma^2)$ and $Z \sim N(0,1)$
with $(m,\sigma)\in \mathbb{R} \times \mathbb{R}_+$. \\

Now take the following parametrization of  the discriminator and the generator networks:
\begin{equation}
\label{Eq:simple_model}
\left\{
\begin{array}{ll}
D_w(x) & = D_{(w_1,w_2,w_3)}(x) = \cfrac{1}{1+e^{-(w_3/2\cdot x^2 + w_2 x + w_1 )}},\\
G_\theta (z) & = G_{(\theta_1,\theta_2)}(z) = \theta_2 z + \theta_1,
\end{array}
\right.
\end{equation}
where $w = (w_1,w_2,w_3) \in \mathbb{R}^3$, and $\theta = ( \theta_1, \theta_2) \in \mathbb{R} \times \mathbb{R}_+$. 
Note that this parametrization of the discriminator and the generator networks are standard since the generator is a simple linear layer and the discriminator can be seen as the composition of the sigmoid activation function with a linear layer in the variable $(x,x^2)$.\\

To find the optimal choice for the parameters $w$ and $\theta$,  denote by $f_X$ and $f_G$ respectively the density functions of $X$ and $G_\theta(Z)$. Then, the density $f_{G^*}$ of the optimal generator is given by $f_{G^*} = f_X$,  meaning  $\theta^* = (m,\sigma)$. Moreover, Proposition 1 in \citep{goodfellow2014generative} shows that the optimal value for the discriminator is
\begin{align*}
D_{w^*}(x) = \cfrac{f_X(x)}{f_X(x) + f_{G^*}(x)} = \cfrac{1}{1 + (f_{G^*}/f_X)(x)} = \cfrac{1}{2},
\end{align*}
for any $x\in \mathbb{R}$ which gives $w^* = (0,0,0)$. \\

We now demonstrate that  the function $g$ may not satisfy the convexity/concavity requirement.
\begin{itemize}
\item To see this, let us first study the concavity of the function $g_{\theta^0}:w \rightarrow g(w,\theta^0)$ with $\theta^0$ fixed. We will show that $g_{\theta^0}$ is  concave with respect to $w$. 

To this end, write $D_w$ as the composition of the two following functions $D^1$ and $L$:
\begin{equation*}
\hspace{-0.2cm}D^1(x) = 1/(1+e^{-x}), \quad L(w;x) = w_3/2\cdot x^2 + w_2 x + w_1,
\end{equation*}
for any $x \in \mathbb{R}$, and $w =(w_1,w_2,w_3) \in \mathbb{R}^3$. Note that a straightforward computation of the second derivatives shows that the functions 
\begin{equation*}
g^1: x \rightarrow \log(D^1(x)), \quad  g^2: x \rightarrow \log(1 - D^1(x)),
\end{equation*} 
with $x \in \mathbb{R}$ are both concave. Thus, by linearity and composition, the function $g_{\theta^0}:w \rightarrow  \Esp_X[g^1(L(w;X))] + \Esp_Z[ g^2(L(w;G_{\theta^0}(Z)))]$ remains concave.

\item Next, we investigate the convexity of the function $g_{w^0}:\theta \rightarrow g(w^0,\theta)$ with $w^0$ fixed. We will show that $g_{w^0}$ is not necessarily convex with respect to $\theta$. \\

To this end, first note  that the mapping $\theta : \rightarrow \Esp_X[\log(D_w(X))]$ does not depend on the parameter $\theta$. Therefore, one can simply focus on the function $g^3:\theta  \rightarrow \Esp_{Z}[\log(1 - D_{w^0}(G_{\theta}(Z)))]$, which is not necessarily convex with respect to $\theta$.
To see this,  let us take $\theta_2 = 0$ for simplicity\footnote{This enables us to get rid of the expectation with respect to $Z$.} and study the function
\begin{align*}
g^3|_{\theta_2=0}: \theta_1 & \rightarrow \log(1 - D_{w^0}(G_{(\theta_1,0)}(z))) \\
            & = \log\big(1 - 1/(1 + e^{-(w^0_3/2\cdot \theta_1^2 + w^0_2 \theta_1 + w^0_1)})\big),
\end{align*}
with $\theta_1 \in \mathbb{R}$. On one hand, the convexity of $g^3|_{\theta_2=0} $ depends on the choice of the parameter $w^0_3$, as demonstrated in Figure \ref{Fig:Plot_phi31}.  On the other hand, simple computation of the second derivative gives 

$$
(g^3|_{\theta_2=0})^{(2)}(\theta_1) =\\
-\cfrac{e^{h(\theta_1) +w_1} \big(w_3 (e^{h(\theta_1) +w_1} +2 h(\theta_1)  + 1) + w_2^2\big) }{((e^{h(\theta_1) +w_1} + 1)^2}, $$

with $\theta_1 \in \mathbb{R}$, and $h(\theta_1) = \theta_1 ( w^0_3/2 \times \theta_1 + w^0_2) $. The sign of this function depends on the choice of the parameter $w^0$. Thus, $g_{w^0}$ is not necessarily convex with respect to $\theta$.
\end{itemize}   

\begin{figure}[h!]
    \center
	 (a) \hspace{3.5cm} (b)\\
    \includegraphics[width=0.23\textwidth]{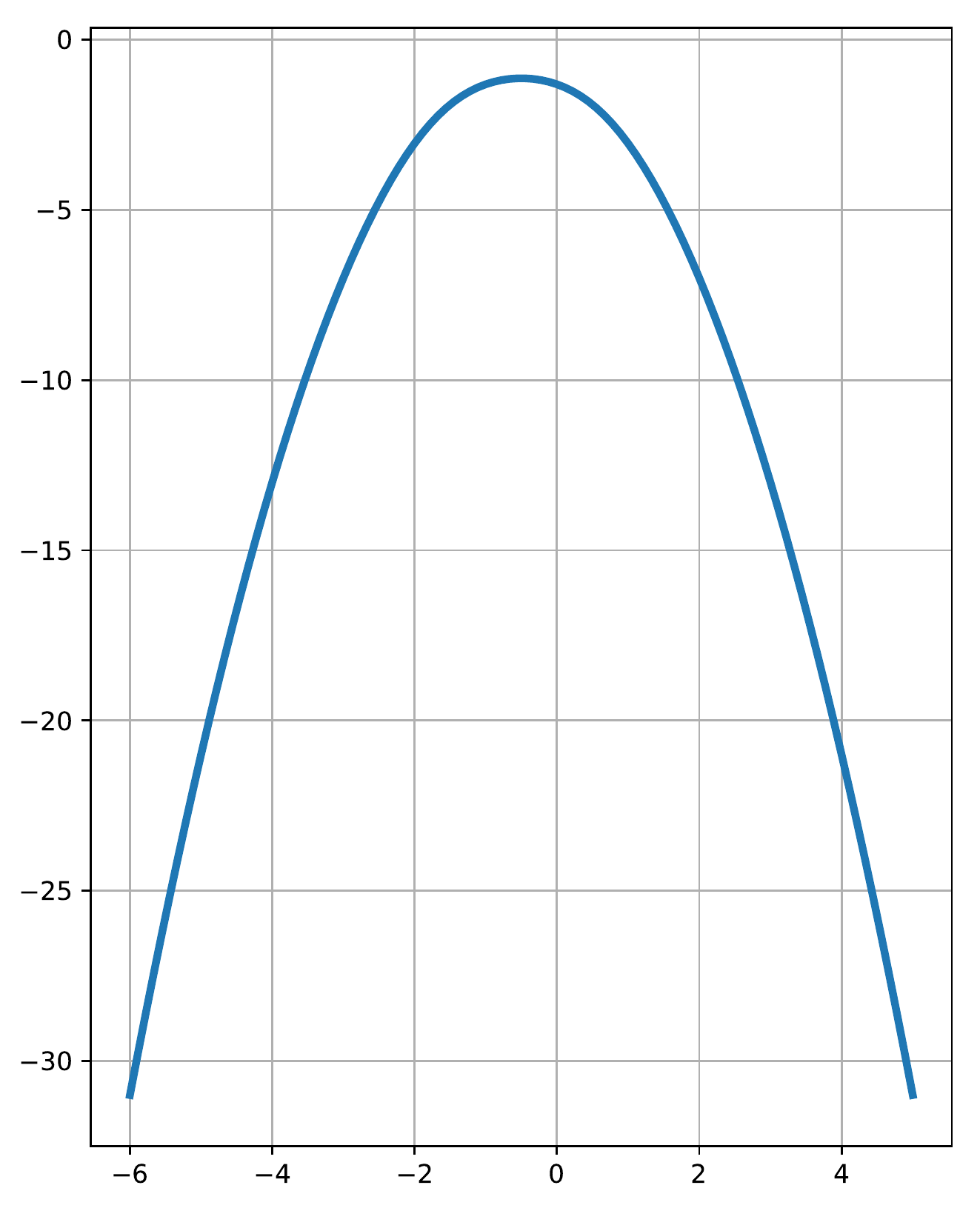}
    \includegraphics[width=0.2325\textwidth]{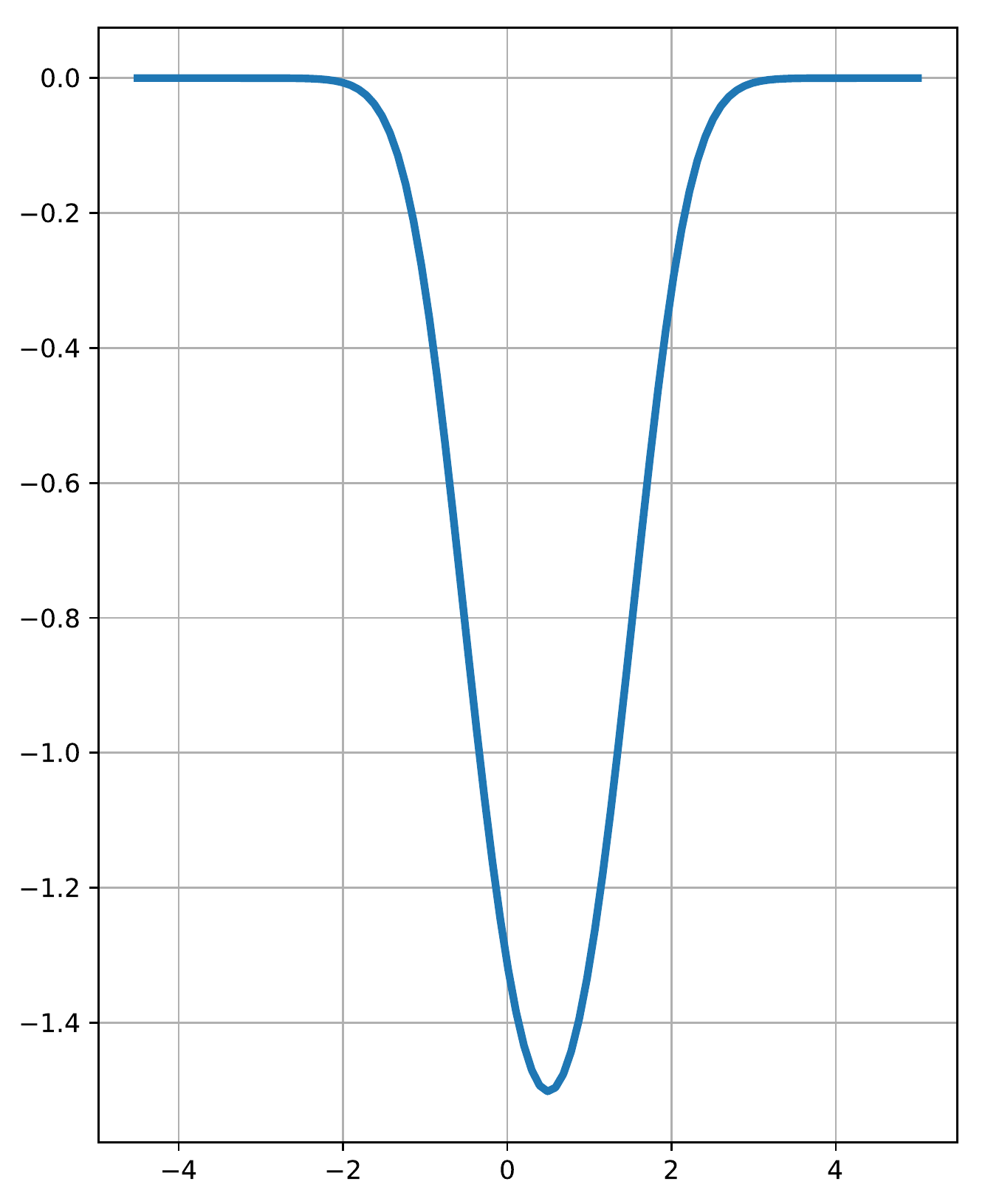}
    \caption{Plot of $g^3|_{\theta_2=0}$ with $(w_1,w_2,w_3) = (1,1,2)$ in (a) and $(w_1,w_2,w_3) = (1,1,-2)$ in (b).}
    \label{Fig:Plot_phi31}
\end{figure}

The analysis indicates analytically one of the culprits behind the convergence issue of GANs: the lack of convexity in the objective function, hence the well-posedness problem of GANs models.

\subsection{GANs Training and Parameters Tuning}
\manuallabel{sec:GANTrainParamTun}{2.2}

In addition to the convexity and well-posedness issue of GANs as minimax games, appropriate choices and specifications for parameters' tuning affect GANs training as well.

\paragraph{Hyper-parameters in GANs training.} There are three key hyper-parameters in stochastic gradient algorithms for GANs training: the learning rate, the batch size, and the time scale.

\begin{itemize}
\item The learning rate determines how far to move along the gradient direction. A higher learning rate accelerates the convergence, yet with a higher chance of explosion; while a lower learning rate yields slower convergence.

\item The time scale parameters monitor the number of updates of the variables $w$ and $\theta$. In the context of GANs training, there are generally two different time scales: a finer time scale for the discriminator and a coarser one for the generator or conversely. Note that more updates means faster convergence, however, also computationally more costly. 

\item The batch size refers to the number of training samples used in the estimate of the gradient.
The more training samples used in the estimate, the more accurate this estimate is, yet with a higher computational cost. Smaller batch size, while providing a cruder estimate, offers a regularizing effect and lowers the chance of overfitting. 
\end{itemize}

\paragraph{Example of improper learning rate.}
 A simple example below demonstrates the importance of an appropriate  choice of learning rate  for the  convergence of SGA. Consider the $\mathbb{R}$-valued function $$f(x) = (a /2) \,\, x^2 + b \, x, \qquad \forall x \in \mathbb{R},$$  where $(a,b) \in \mathbb{R}_{+} \times \mathbb{R}$.   Finding the minimum $x^* = -(b/a)$ of $f$ via the gradient algorithm consists of updating an initial guess $x_0 \in \mathbb{R}$ as follows:
\begin{align}
x_{n+1} = x_{n} - \eta (a x_{n} + b ),\qquad \forall n \geq 0, \label{Eq:GradEgExplo}
\end{align}
with $\eta$ the learning rate. Let us study the behavior of the error $e_n = |x_n - x^*|^2$. By \eqref{Eq:GradEgExplo} and $ a x^* + b = 0$,  
\begin{align*}
e_{n+1} = |x_{n+1} - x^*|^2 & = |x_n - x^*|^2 + 2 ( x_{n+1} - x_n ) e_n +  |x_{n+1} - x_n|^2 \\
                  & = \big(1 - \eta a(2  - \eta a ) \big)|x_{n} - x^*|^2. \numberthis \label{Eq:GradEgExplo2}
\end{align*}
Thus, when $\eta > 2/a$, the factor $\eta a(2  - \eta a ) < 0 $ which means $ r = \big(1 - \eta a(2  - \eta a ) \big) > 1$. In such a case, Equation \eqref{Eq:GradEgExplo2} becomes
$$
e_{n+1} = r \, e_n,
$$
ensuring $e_n \rightarrow + \infty$ when $n$ goes to infinity, and leading to the failure of the convergence for the gradient algorithm. This example highlights the importance of the learning rate parameter. Such an issue of improper learning rates for GANs training will be revisited in a more general setting (see Section \ref{sec:OptLearningRate}). \\

Note that there are earlier works on optimal learning rate policies to improve the performance of gradient-like algorithms (see for instance \citep{moulines2011non}, \citep{gadat2017optimal}, and \citep{mounjid2019improving}).

\paragraph{Time scale for GANs training.} Equation \eqref{Eq:update_rule_grad} corresponds to a specific implementation where updates of the parameters $w$ and $\theta$ are alternated. However, it also is possible to consider an asynchronous update of the following form (see for example Algorithm 1 in \citep{arjovsky2017wasserstein}):

\begin{algorithm}[H]
\caption{Asynchronous gradient algorithm}
\begin{algorithmic}[1]
\For{$i = 1 \ldots n^{\theta}_{\max}$}
	\For{$j = 1 \ldots n^{w}_{\max}$}
			\State $w \leftarrow w  + \eta^{w} g_{w}(w,\theta) $
	\EndFor	
	\State $\theta \leftarrow \theta - \eta^{\theta} g_{\theta}(w,\theta)$
\EndFor
\end{algorithmic}
\label{Alg:async_update_rule_sde}
\end{algorithm}

where $n^{w}_{\max}$ and $n^{\theta}_{\max}$ are respectively the maximum number of iterations for the upper and the inner loop. In such a context, one naturally deals with two time scales: a finer time scale for the discriminator and a coarser one for the generator. The time scale parameters $n^{w}_{\max}$ and $n^{\theta}_{\max}$ monitor the number of updates of the variables $w$ and $\theta$. As mentioned earlier, more updates means faster convergence but at the cost of more gradient computations. It is therefore necessary to select these parameters carefully in order to perform updates only when needed.

\paragraph{Batch size for GANs training.} 

To better understand the batch size impact, we consider here the vanilla objective function $g$, i.e., 
\begin{equation*}
g(w,\theta) = \Esp_X\big[\log(D_w(X))\big] + \Esp_Z\big[\log\big(1 - D_w(G_\theta(Z))\big)\big].
\end{equation*}
In general, the function $g$ is approximated by the empirical mean 
\begin{equation}
g^{NM}(w,\theta) = \cfrac{\sum_{i = 1}^N \sum_{j = 1}^M g^{i,j}(w,\theta)}{N \cdot M},
\label{Eq:empMeanEstimate}
\end{equation}
with 
$$
g^{i,j}(w,\theta) = \log( D_{w}(x_i)) + \log\big(1 - D_{w}\big(G_{\theta}(z_j)\big)\big),
$$ 
where $(x_i)_{i \leq N}$ are i.i.d. samples from $\mathbb{P}_X$ the distribution of $X$, $(z_j)_{j \leq M}$ are i.i.d. samples from $\mathbb{P}_Z$ the distribution of $Z$, and $N$ (resp. $M$) is the number of $\mathbb{P}_X$ (resp. $\mathbb{P}_Z$) samples. The quantity $N\cdot M$ here represents the batch size. It is clear from Equation \eqref{Eq:empMeanEstimate} that enlarging the batch size offers a better estimate of $g^{NM}(w,\theta)$ since it reduces the variance. However, such an improvement requires a higher computational power. Sometimes, it is better to spend such a power on performing more gradient updates rather than reducing the variance.

\section{Control and Game Formulation of GANs Training}

Clearly, hyper-parameters introduced in Section \ref{sec:GANTrainParamTun} are not independent for GANs training, which often involves choices between adjusting learning rates and changing sample sizes. In this section, we show how hyper-parameters' tuning can be formulated and analyzed as stochastic control problems, and how the popular practice of ``clipping'' in GANs training can be understood analytically in this framework. 

\subsection{Stochastic Control of Learning Rate}
\manuallabel{sec:LRAnalysis}{3.1}

In this part, we present an optimal selection methodology for the learning rate.  To start, let us recall the continuous-time stochastic differential equation used to represent GANs training.

\paragraph{SDE approximation of GANs training.} In GANs training, gradient algorithms for  the optimal parameters $\theta^*$ and $w^*$ start with an initial guess $(w_0,\theta_0)$ and apply at time step $t$ the following (simultaneous) update rule:
\begin{equation}
\left\{
\begin{array}{ll}
w_{t+1} & = w_{t} + \eta g_{w}(w_t,\theta_t), \\
\theta_{t+1} & = \theta_t - \eta g_{\theta}(w_t,\theta_t),
\end{array}
\right.
\label{Eq:update_rule_grad}
\end{equation}  
with $g_{w} = \nabla_{w} g$,  $g_{\theta} = \nabla_{\theta} g $, and $ \eta \in \mathbb{R}_+$  the learning rate. \\

The continuous-time approximation of GANs training via functional central limit theorem \citep{cao2020approximation} replaces the update rule in \eqref{Eq:update_rule_grad} with coupled
stochastic differential equations (SDEs)
\begin{equation}
\left\{
\begin{array}{ll}
dw(t) & =  g_{w}(q(t)) dt + \sqrt{\eta} \sigma_w(q(t)) dW^1(t), \\
d\theta(t) & = -  g_{\theta}(q(t)) dt + \sqrt{\eta} \sigma_\theta(q(t)) dW^2(t),
\end{array}
\right.
\label{Eq:sde}
\end{equation}
where $q(t) = (w(t),\theta(t))$, the functions $\sigma_w: \mathbb{R}^{M} \times \mathbb{R}^{N} \rightarrow \mathcal{M}_{\mathbb{R}}(M)$ and $\sigma_\theta: \mathbb{R}^{M} \times \mathbb{R}^{N} \rightarrow \mathcal{M}_{\mathbb{R}}(N)$ are approximated by the covariances of $g_{w}$ and $g_{\theta}$, and  the Brownian motions $W^1$ and $W^2$ are independent. Note that in this SDE approximation, the learning rates for  the generator and the discriminator   $(\eta,\eta)$ are fixed constants. Based on this approximation, optimal choices of learning rate can be formulated as a stochastic control problem.

\paragraph{Adaptive learning rate.} The idea goes as follows. Consider the following decomposition for the learning rate $\eta(t)$ at time $t$:
\begin{align*}
\eta(t) = \big( \eta^w(t), \eta^{\theta}(t)\big) & = \big( u^w(t) \times \bar{\eta}^w(t), u^\theta(t) \times \bar{\eta}^{\theta}(t)\big) =  u(t) \bullet \bar{\eta}(t), \quad \forall t \geq 0,
\numberthis \label{eq:learning_rate}
\end{align*}
where the first component $\bar{\eta}(t) = (\bar{\eta}^w(t), \bar{\eta}^{\theta}(t))\in [\bar{\eta}^{\min},1]^2$ is a predefined base learning rate fixed by the controller using her favorite learning rate selection method, and the second component $u(t) = (u^w(t), u^{\theta}(t))$ is a $[u^{\min},u^{\max}]^2$-valued process representing an adjustment around $\bar{\eta}(t)$. The symbol $\bullet $ here is the component-wise product between vectors. Furthermore, the positive constants $\bar{\eta}^{\min}$, $u^{\min}$, and $u^{\max}$ are ``clipping'' parameters introduced to handle the convexity issue discussed earlier for GANs training, to establish ellipticity conditions needed for the regularity of the value function, and to avoid explosion. This explosion aspect will be clarified shortly. With the incorporation of the adaptive learning rate \eqref{eq:learning_rate}, the corresponding SDE for GANs training becomes
\begin{equation}
\left\{
\begin{array}{l}
dw(t) =  u^{w}(t) g_{w}(q(t)) dt + \big(u^{w}\sqrt{\bar{\eta}^{w}}\big)(t) \sigma_w(q(t)) dW^1(t), \\
\\
d\theta(t) = -  u^{\theta}(t) g_{\theta}(q(t)) dt + \big(u^{\theta}\sqrt{\bar{\eta}^{\theta}}\big)(t) \sigma_\theta(q(t)) dW^2(t),
\end{array}
\right. \label{Eq:sde2}
\end{equation}
with $q(t) = (w(t),\theta(t)) $ for any $t\geq 0$. 

\paragraph{Optimal control of adaptive learning rate.} Let $T< \infty$ be a finite time horizon, and define the objective function $J$ as
\begin{equation*}
J(T,t,q;u) = \Esp\big[g\big(q(T)\big)\big|q(t) = q],
\end{equation*}
where $q = (w,\theta) \in \mathbb{R}^{M} \times \mathbb{R}^{N}$ is the value of the process $q(t)$ at $t \in [0,T]$. Note that the function $J$ here is similar to the mapping $g$ used in vanilla GANs (see Equation \eqref{Eq:min_pbm_gan}). The main difference consists of replacing the constant parameters $(w,\theta)$ in \eqref{Eq:min_pbm_gan} by $q(T)$ their value at the end of the training. Moreover, since $q(T)$ is a random variable, an expectation is added to estimate the average value of $g\big(q(T)\big)$. \\

Then, the control problem for the adaptive learning rate is formulated as
\begin{equation}
v(t,q) = \min_{u^\theta \in \mathcal{U}^{\theta}} \max_{u^{w} \in \mathcal{U}^{w}} J(T,t,q;u),
\label{Eq:ValFunctDef}
\end{equation}
for any $(t,q) \in [0,T] \times \mathbb{R}^M \times \mathbb{R}^N$, with $\mathcal{U}^{w}$ and $\mathcal{U}^{\theta}$ the set of appropriate admissible controls for $u^w$ and $u^\theta$. For a fixed $u^\theta \in \mathcal{U}^{\theta}$, $\mathcal{U}^{w}$ is defined as
\begin{align*}
\mathcal{U}^{w} = \big\{ u :\;& u \text{ c\`{a}dl\`{a}g in } [u^{\min},u^{\max}] \text{ adapted to } \mathbb{F}^{(W^1,W^2)},\\ 
					      &  \Esp[g\big( q(T) \big) \,|\,q(0)\,] < \infty \big\},
\end{align*}
where $u^{\max} \geq u^{\min} > 0$ are the upper bounds introduced earlier. Then, we write $\mathcal{U}^{\theta}$ as follows:
\begin{align*}
\mathcal{U}^{\theta} = \big\{ u :\;& u \text{ c\`{a}dl\`{a}g in } [u^{\min},u^{\max}] \text{ adapted to } \mathbb{F}^{(W^1,W^2)},\\ 
					      &  \sup_{u \in \mathcal{U}^{w}} \Esp[g\big( q(T) \big) \,|\,q(0)\,] < \infty
					      \big\}.
\end{align*}

\subsection{Stochastic Control of Time Scales}
\manuallabel{sec:TimeScaleSelec}{3.2}

Let us consider the following expression for $n_{\max}$:
\begin{align*}
n_{\max} = (n^{w}_{\max},n^{\theta}_{\max}) = \big( c^w \times \bar{n}^w_{\max}, c^\theta \times \bar{n}^\theta_{\max}\big) = c \bullet \bar{n}_{\max},
\end{align*}
with $\bar{n}_{\max} = (\bar{n}^{w}_{\max},\bar{n}^{\theta}_{\max})$ a base time scale parameter initially fixed by the controller and $c = (c^{w},c^{\theta})$ a constant adjustment around $\bar{n}_{\max}$. Under suitable conditions (see for example \citep{fatkullin2004computational, weinan2005analysis}) one can show that the asynchronous Algorithm \ref{Alg:async_update_rule_sde} converges towards
\begin{equation}
\left\{
\begin{array}{l}
dw(t) =  \cfrac{1}{c^w \epsilon^1} g_{w}(q(t)) dt , \\
d\theta(t) = -  \cfrac{1}{c^\theta}  g_{\theta}(q(t)) dt,
\end{array}
\right. 
\label{Eq:pde1}
\end{equation}
where $q(t) = (w(t) , \theta(t))$ and $\epsilon^1$ is a small parameter measuring the separation between time scales. Thus, the SDE version of \eqref{Eq:pde1} is
\begin{equation*}
\hspace{-0.3cm}
\left\{
\begin{array}{l}
dw(t) = \tilde{\eta}^{w} g_{w}(q(t)) dt +  \tilde{\eta}^{w} \sqrt{\bar{\eta}} \sigma_w(q(t)) dW^1(t), \\
\\
d\theta(t) = -  \tilde{\eta}^{\theta} g_{\theta}(q(t)) dt + \tilde{\eta}^{\theta} \sqrt{\bar{\eta}}  \sigma_\theta(q(t)) dW^2(t),
\end{array}
\right. 
\end{equation*}

with $\tilde{\eta}^w = 1/(c^w \epsilon^1)$, and $\tilde{\eta}^\theta = 1/c^\theta$. Comparing the dynamics of $(w(t),\theta(t))_{t \geq 0}$ with the one of Section \ref{sec:LRAnalysis} suggests that {\it the time scale and the learning rate control problems are equivalent.}

\subsection{Stochastic Control of Batch Size}
\manuallabel{sec:BatchSizeSelec}{3.3}

To understand the impact of the batch size, let us introduce a scaling factor $m^\theta \geq 1$, the terminal time $T$, a maximum number of iterations $t_{\max}$, and compare the two following algorithms:
\begin{equation*}
\hspace{-1.3cm}
\left\{
\begin{array}{ll}
w^1_{t+1} & = w^1_{t} + \eta g^{NM}_{w}(w^1_t,\theta^1_t), \\
\theta^1_{t+1} & = \theta^1_t - \eta g^{NM}_{\theta}(w^1_t,\theta^1_t),
\end{array}
\right.
\quad \forall t \leq t_{\max},
\end{equation*}  
and
\begin{equation*}
\left\{
\begin{array}{ll}
w^2_{t+1} & = w^2_{t} +  \eta g^{m^\theta (NM)}_{w}(w^2_t,\theta^2_t), \\
\theta^2_{t+1} & = \theta^2_t - \eta g^{m^\theta (NM)}_{\theta}(w^2_t,\theta^2_t),
\end{array}
\right.
\quad \forall t \leq  t_{\max}/m^\theta.
\end{equation*}
Note that the second algorithm is simulated less to ensure computational costs for the two methods are the same, i.e., with the same number of gradient calculations. Following \citep{cao2020approximation} the continuous-time approximation  for both algorithms can be written as
\begin{equation*}
\hspace{-0.7cm}
\left\{
\begin{array}{l}
dw^{1}(t) =  g_{w}(q^1(t)) dt + \sqrt{\eta} \sigma_w(q^1(t)) dW^1(t), \\
\\
d\theta^1(t) = -  g_{\theta}(q^1(t)) dt + \sqrt{\eta} \sigma_\theta(q^1(t)) dW^2(t),
\end{array}
\right. 
\end{equation*}
and 
\begin{equation*}
\left\{
\begin{array}{l}
dw^{2}(t) =  g_{w}(q^2(t)) dt + \sqrt{\eta/m^\theta}  \sigma_w(q^2(t)) dW^1(t), \\
\\
d\theta^2(t) = -  g_{\theta}(q^2(t)) dt + \sqrt{\eta/m^\theta}  \sigma_\theta(q^2(t)) dW^2(t),
\end{array}
\right. 
\end{equation*}
with $q^1(t) = (w^1(t),\theta^1(t))$, $q^2(t) = (w^2(t),\theta^2(t))$, $g_w = \Esp[g_w^{NM}]$, $g_\theta = \Esp[g_\theta^{NM}]$, $\sigma_w^2$ (resp. $\sigma_\theta^2$) proportional to the variance of $g_w^{NM}$ (resp. $g_\theta^{NM}$), and $\eta \geq 0$ a constant learning rate. Note that, since samples are i.i.d., the variances of $g_w^{m^\theta NM}$ and $g_\theta^{m^\theta NM} $ are $m^\theta$ times smaller than $g_w^{NM}$ and $g_\theta^{NM} $ variances, i.e., $\Var(g_w^{NM}) = m^\theta \Var(g_w^{m^\theta NM})$, implying that enlarging the batch size reduces the variance.\\

Meanwhile, comparing the objective functions $\Esp[g(q^{1}(T))]$ and $\Esp[g(q^{2}(T/m^\theta))]$ of both implementations suggests that reducing the batch size leads to a larger time horizon which means more parameters updates. Therefore, the question of finding the right trade-off between reducing the variance and performing more updates arises naturally.\\

Now, one  can select the optimal batch size in the same spirit of the learning rate control problem (see Section \ref{sec:LRAnalysis}). That is to consider 
\begin{equation}
\tilde{v}^m(t,q) = \min_{m^\theta \in \mathcal{M}^{\theta}}  \Esp\big[g\big(\tilde{q}^m(T)\big)\big|\tilde{q}^m(t) = q],
\label{Eq:ValFunctDef23}
\end{equation}
for any $(t,q) \in \mathbb{R}_+ \times \mathbb{R}^{M+N}$, where the process $\tilde{q}^m(t) = (\tilde{w}^m(t), \tilde{\theta}^m(t))$ represents the trained parameters at time $t$, and $\mathcal{M}^{\theta}$ is the set of admissible controls for $m^{\theta}$. The process $\tilde{q}^m$ satisfies the SDE below
\begin{equation}
\hspace{-0.3cm}
\left\{
\begin{array}{l}
d\tilde{w}^m(t) =   \cfrac{g_{w}(\tilde{q}^m(t))}{m^\theta} dt + \cfrac{\sqrt{\eta} \sigma_w(\tilde{q}^m(t))}{m^\theta } d\tilde{W}^1(t), \\
\\
d\tilde{\theta}^m(t) = -  \cfrac{ g_{\theta}(\tilde{q}^m(t))}{m^\theta } dt + \cfrac{\sqrt{\eta} \sigma_\theta(\tilde{q}^m(t))}{m^\theta } d\tilde{W}^2(t),
\end{array}
\right.
\label{Eq:sde23}
\end{equation}
with $\tilde{W}^1$ and $\tilde{W}^2$ two independent Brownian motions. The derivation of \eqref{Eq:sde23} is detailed in Appendix \ref{Appendix:proofOfSDE23}. Moreover, the set $\mathcal{M}^{\theta}$ is defined as
\begin{align*}
\mathcal{M}^{\theta} = \big\{ m :\;& m \text{ c\`{a}dl\`{a}g in } [1,m^{\max}] \text{ adapted to } \mathbb{F}^{(\tilde{W}^1,\tilde{W}^2)},\\ 
					      &   \Esp[g\big( \tilde{q}^m( T \big) \,|\,\tilde{q}^m(0)\,] < \infty
					      \big\}.
\end{align*}
We allow here $m^\theta$ to be a process in order to handle a more general control problem. 

\section{Stochastic Differential Games}
\label{sec:St_Game_ctrl_pbm}

The minimax game of GANs with adaptive learning rate and batch size in the previous section can be analyzed in a more general framework of stochastic differential games. In this section, we first establish a weak form of dynamic programming principle (DPP) for a class of stochastic differential games where the underlying process is not necessarily a controlled Markov diffusion and where the value function is not a priori continuous. We will then apply this weak form of dynamic programming principle to stochastic games with controlled Markov diffusion, and show that the value of such games is the unique viscosity solution to the associated Isaac-Bellman equation, under suitable technical conditions.
		
\subsection{Formulation of Stochastic Differential Games}
\manuallabel{subsec:form_ctrl_pbm}{8.1}

Let $d \geq 1$ be a fixed integer and $(\Omega,\mathcal{F},\mathbb{P})$ be a probability space supporting a c\`{a}dl\`{a}g $\mathbb{R}^d$-valued process $Y$ with independent increments. Given $T \in \mathbb{R}_{+}^*$, we write $\mathbb{F} = \{\mathcal{F}_t,\,0 \leq t \leq T \}$ for the completion of $Y$ natural filtration on $[0,T]$. Here $\mathbb{F}$ satisfies the usual condition (see for instance \citep{jacod2013limit}). We suppose that $\mathcal{F}_0$ is trivial and that $\mathcal{F}_T = \mathcal{F}$. Moreover, for every $t \geq 0$, set $\mathbb{F}^t = \{\mathcal{F}^t_s,\,s \geq 0 \}$ where $\mathcal{F}^t_s$ is the completion of $\sigma(Y_r - Y_t,\, t \leq r \leq s \vee t )$ by null sets of $\mathcal{F}$.\\

The set $\mathcal{T}$ refers to the collection of all $\mathbb{F}$-stopping times. For any $(\tau_1,\,\tau_2) \in \mathcal{T}^2$ such that $\tau_1 \leq \tau_2$, the subset $\mathcal{T}_{[\tau_1,\tau_2]}$ is the collection of all $\tau \in \mathcal{T}$ verifying $\tau \in [\tau_1,\tau_2]\, a.s$. When $\tau_1 = 0$, we simply write $\mathcal{T}_{\tau_2} $. We use the notations $\mathcal{T}^{t}_{[\tau_1,\tau_2]}$ and $\mathcal{T}^t_{\tau_2} $ to denote the corresponding sets of $\mathbb{F}^t$-stopping times.\\

For every $ \tau \in \mathcal{T}$ and a subset $A$ of a finite-dimensional space, we denote by $L^0_\tau(A)$ the collection of all $\mathcal{F}_\tau$-measurable random variables with values in $A$. The set $\mathbb{H}^0(A)$ is the collection of all $\mathbb{F}$-progressively measurable processes with values in $A$, and $\mathbb{H}^0_{rcll}(A) $ is the subset of all processes in $\mathbb{H}^0(A)$ which are right continuous with finite left limits. We first introduce the sets 
\begin{align*}
\mathbb{S} = [0,T] \times \mathbb{R}^d, \quad \mathcal{S}_0 = \{(\tau,\epsilon);\,\tau \in \mathcal{T}_T,\, \epsilon \in L^0_\tau(\mathbb{R}^d) \}.
\end{align*}
Then we take the two sets of control processes $\mathcal{U}_0^1 \subset \mathbb{H}^0(\mathbb{R}^{k^1})$ and $\mathcal{U}_0^2 \subset \mathbb{H}^0(\mathbb{R}^{k^2}) $, with $k^1 \geq 1$ and $k^2 \geq 1$ two integers, such that the controlled state process defined as the mapping 
\begin{equation*}
(\tau,\epsilon;\nu^1,\nu^2) \in \mathcal{S} \times \mathcal{U}_0^1 \times \mathcal{U}_0^2 \longrightarrow X^{(\nu^1,\nu^2)}_{\tau,\epsilon},
\end{equation*}
is well defined and 
\begin{equation*}
(\theta,X^{\nu}_{\tau,\epsilon}(\theta))\in \mathcal{S}, \qquad \forall(\tau,\epsilon)\in \mathcal{S},\forall \theta \in \mathcal{T}_{[\tau,T]}.
\end{equation*}
Here, $X^{(\nu^1,\nu^2)}_{\tau,\epsilon}$ refers to the controlled process and $\mathcal{S}$ is a set satisfying $\mathbb{S} \subset \mathcal{S} \subset \mathcal{S}_0$. For instance, one can take $\mathcal{S} = \{(\tau,\epsilon) \in \mathcal{S}_0;\,\Esp[|\epsilon|^2] < \infty \}$. In the sequel, we write $\mathcal{U}_0$ for the set $\mathcal{U}_0 = \mathcal{U}_0^1 \times \mathcal{U}_0^2$.\\

Let $f:\mathbb{R}^d \rightarrow \mathbb{R}$ be a Borel function, and the reward function $J$ be
\begin{equation}
J(t,x;\nu) = \Esp[f(X^{\nu}_{t,x}(T))], \qquad \forall (t,x)\in \mathbb{S},\forall \nu = (\nu^1,\nu^2) \in \mathcal{U}^1 \times \mathcal{U}^2,\label{Eq:reward_fct}
\end{equation}
with $\mathcal{U}^1$ (resp. $\mathcal{U}^2$) the set of admissible controls for $\nu^1$ (resp. $\nu^2$). Given $\nu^2$, we define $\mathcal{U}^{1}$ as 
\begin{align*}
\mathcal{U}^{1} = \big\{ \nu^1 \in \mathcal{U}_0^1 ;\; & \Esp[|f(X^{(\nu^1,\nu^2)}_{t,x}(T))|] < \infty \big\}.
\end{align*}
Then, we denote by $\mathcal{U}^{2}$ the set
\begin{align*}
\mathcal{U}^{2} = \big\{ \nu^2 \in \mathcal{U}_0^2 :\; & \sup_{\nu^1 \in \mathcal{U}^{1}} \Esp[|f(X^{(\nu^1,\nu^2)}_{t,x}(T))|] < \infty \big\}.
\end{align*}
We write $\mathcal{U}^{1}_t$ (resp. $\mathcal{U}^{2}_t$) for the collection of processes $\nu^1 \in \mathcal{U}^{1}$ (resp. $\nu^2 \in \mathcal{U}^{2}$) that are $\mathbb{F}^t$-progressively measurable. We denote by $\mathcal{U}_t$ the set $\mathcal{U}_t = \mathcal{U}^{1}_t \times \mathcal{U}^{2}_t$. The value function $V$ of the stochastic control problem can be written as
\begin{equation*}
V(t,x) =  \inf_{\nu^2 \in \mathcal{U}^{2}_t }\sup_{\nu^1 \in \mathcal{U}^{1}_t } J(t,x;\nu),\qquad \forall (t,x)\in \mathbb{S}.
\end{equation*}

\subsection{Dynamic Programming Principle for Stochastic Differential Games}
\manuallabel{subsec:dpp_ctrl_pbm}{8.2}

To establish the weak form of the dynamic programming principle, we work under the following mild assumptions, as in \citep{bouchard2011weak}.
\begin{Assumption} For all $(t,x) \in \mathbf{S}$, and $\nu \in \mathcal{U}_t$, the controlled state process satisfies
\begin{enumerate}[label= A.\arabic*, leftmargin=0.5cm]
\item  \textbf{Independence.} The process $X^{\nu}_{t,x}$  is $\mathbb{F}^t$-progressively measurable.
\item \textbf{Causality.} For any $\tilde{\nu} \in \mathcal{U}_t$, $\tau \in \mathcal{T}^t_{[t,T]}$, and $ A  \in \mathcal{F}^t_\tau $, if $\nu = \tilde{\nu} $ on $[t,\tau]$ and $\nu \mathbf{1}_A = \tilde{\nu} \mathbf{1}_A$ on $(\tau,T] $, then $X^{\nu}_{t,x} \mathbf{1}_A = X^{\tilde{\nu}}_{t,x} \mathbf{1}_A $.
\item \textbf{Stability under concatenation.} For every $\tilde{\nu} \in \mathcal{U}_t$, and $\theta \in \mathcal{T}^t_{[t,T]}$, we have 
\begin{equation*}
\nu \mathbf{1}_{[0,\theta]}  + \tilde{\nu} \mathbf{1}_{(\theta,T]} \in \mathcal{U}_t.
\end{equation*}
\item \textbf{Consistency with deterministic initial data.} For all $\theta \in \mathcal{T}^t_{[t,T]}$, we have the following:
\begin{enumerate}[label=\alph*.]
\item For $\mathbb{P}-$a.e. $\omega \in \Omega$, there exists $\tilde{\nu}_\omega \in \mathcal{U}_{\theta(\omega)}$ such that
\begin{equation*}
\Esp[f(X^{\nu}_{t,x}(T))|\mathcal{F}_\theta](\omega) = J(\theta(\omega),X^{\nu}_{t,x}(\theta)(\omega);\tilde{\nu}_\omega ).
\end{equation*}
\item For $t \leq s \leq T $, $\theta \in \mathcal{T}^t_{[t,s]}$, $\tilde{\nu} \in \mathcal{U}_s$, and $\bar{\nu}=  \nu \mathbf{1}_{[0,\theta]}  + \tilde{\nu} \mathbf{1}_{(\theta,T]}$, we have 
\begin{equation*}
\Esp[f(X^{\bar{\nu}}_{t,x}(T))|\mathcal{F}_\theta](\omega) = J(\theta(\omega),X^{\nu}_{t,x}(\theta)(\omega);\tilde{\nu} ),\quad \text{for } \mathbb{P}-\text{a.e. } \omega \in \Omega.
\end{equation*}
\end{enumerate}
\end{enumerate}
\label{Assump:A}
\end{Assumption}

We also need the boundedness and regularity assumptions below.
\begin{Assumption} \label{Assump:B}
The value function $V$ is locally bounded.
\end{Assumption}

\begin{Assumption}\label{Assump:C}
The reward function $J(.;\nu)$ is continuous for every $\nu \in \mathcal{U}_0$.
\end{Assumption}

We are now ready to derive the weak form of dynamic programming principle.
\begin{theo} Assume Assumptions \ref{Assump:A},  \ref{Assump:B}, and \ref{Assump:C}.  Then
\begin{itemize}
\item For any function $\phi$ upper-semicontinuous such that $V \geq \phi$, we have 
\begin{equation}
V(t,x) \geq \inf_{\nu^2 \in \mathcal{U}^2_t}\sup_{\nu^1 \in \mathcal{U}^1_t} \Esp[\phi(\theta^\nu,X^\nu_{t,x}(\theta^\nu)].\label{Eq:Theo_DPP1_1}
\end{equation}
\item For any function $\phi$ lower-semicontinuous such that $V \leq \phi$, we have 
\begin{equation}
V(t,x) \leq \inf_{\nu^2 \in \mathcal{U}^2_t} \sup_{\nu^1 \in \mathcal{U}^1_t} \Esp[ \phi(\theta^{\nu},X^\nu_{t,x}(\theta^\nu)]. \label{Eq:Theo_DPP1_2}
\end{equation}
\end{itemize}
\label{Theo:DPP1}
\end{theo}

\begin{proof} {(For Theorem \ref{Theo:DPP1}).}
Let us first prove \eqref{Eq:Theo_DPP1_1}. For any admissible process $\nu^2 \in \mathcal{U}^2_t$, define $\bar{V}$ as  
\begin{equation}
\bar{V}(t,x;\nu^2) = \sup_{\nu^1 \in \mathcal{U}^1_t} J(t,x;(\nu^1,\nu^2)), \qquad \forall (t,x) \in \mathbb{S}.
\label{Eq:theo_1_eq0}
\end{equation}
Under Assumptions \ref{Assump:A}, \ref{Assump:B}, and \ref{Assump:C}, one can apply \cite[Theorem 3.5]{bouchard2011weak} to get 
\begin{align*}
\bar{V}(t,x;\nu^{2}) \geq \sup_{\nu^1 \in \mathcal{U}^1_t}\,\Esp[\phi(\theta^\nu,X^{\nu}_{t,x}(\theta^\nu) ].
\numberthis \label{Eq:theo_1_eq1}
\end{align*}
Since, $V(t,x) = \inf_{\nu^2 \in \mathcal{U}^2_t}\,\bar{V}(t,x;\nu^{2})$ by definition, we use \eqref{Eq:theo_1_eq1} to deduce that 
\begin{align*}
V(t,x) \geq \inf_{\nu^2 \in \mathcal{U}^2_t} \sup_{\nu^1 \in \mathcal{U}^1_t}\,\Esp[\phi(\theta^\nu,X^{\nu}_{t,x}(\theta^\nu) ].
\end{align*} 
\item We move now to the proof of \eqref{Eq:Theo_DPP1_2}. Fix $\nu = (\nu^1,\nu^2) \in \mathcal{U}_t$ and set $\theta \in \mathcal{T}^t_{[t,T]}$. For any $t \geq 0$, we write $\bar{t}$  for the time $\bar{t} = t \wedge \theta(\omega)$. Let $\epsilon > 0$ and $\bar{V}$ be the function introduced in \eqref{Eq:theo_1_eq0}.  Then, there is a family of $(\nu^{(s,y),\epsilon,2})_{(s,y)\in \mathbb{S}} \subset \mathcal{U}_0 $ such that 
\begin{equation}
\nu^{(s,y),\epsilon,2} \in \mathcal{U}^2_{s}, \qquad J(s,y;(\nu^{s,1},\nu^{(s,y),\epsilon,2})) - \epsilon \leq \bar{V}(s,y;\nu^{(s,y),\epsilon,2}) - \epsilon \leq  V(s,y), \label{Eq:theo_1_eq2}
\end{equation}
for any $(s,y) \in \mathbb{S}$ and $\nu^{s,1} \in \mathcal{U}^1_{s}$. Let $\nu^{(s,y),\epsilon} = (\nu^{1},\nu^{(s,y),\epsilon,2}) \in \mathcal{U}_{s}$. Since $\phi$ is lower-semicontinuous, and $J$ is upper-semicontinuous, there exists a family $(r_{(s,y)})_{(s,y) \in \mathbb{S}}$ of positive scalars such that
\begin{align}
\phi(s,y) - \phi(s',y') \leq \epsilon, \qquad J(s,y;\nu^{(s,y),\epsilon}) - J(s',y';\nu^{(s,x),\epsilon}) \geq -\epsilon, \quad \forall (s',y') \in B(s,y;r_{(s,y)}), \label{Eq:theo_1_eq3}
\end{align}
with $(s,y) \in \mathbb{S}$ and
\begin{equation*}
B(s,y;r) = \{(s',y')\in \mathbb{S}; \, s' \in (s-r,s],\, \|y-y'\| < r \}, \qquad \forall r > 0. 
\end{equation*} 
Now follow the same approach of \cite[Theorem 3.5, \emph{step} $2$]{bouchard2011weak} to construct a countable sequence $(t_i,y_i,r_i)_{i \geq 1}$ of elements of $\mathbb{S} \times \mathbb{R}$, with $0 < r_i \leq r_{(t_i,y_i)}$ for all $ i \geq 1$, such that $\mathbb{S} \subset \{0\}\times \mathbb{R}^d \cup \{\cup_{i \geq 1} B(t_i,y_i;r_i)\} $. Set $A_0 = \{T\} \times \mathbb{R}^d$, $C_{-1} = \emptyset$, and define the sequence 
$$ 
A_{i+1} = B(t_{i+1},y_{i+1};r_{i+1})\setminus C_i, \qquad C_i = C_{i-1} \cup A_i, \quad \forall i \geq 0.
$$
With this construction, it follows from \eqref{Eq:theo_1_eq2}, \eqref{Eq:theo_1_eq3}, and the fact that $V \leq \phi$, that the countable family $(A_i)_{i \geq 0}$ satisfies 
\begin{align*}
\left\{ 
\begin{array}{ll}
\big( \theta, X^{\nu}_{t,x} ( \theta ) \big) \in \big(\cup_{i \geq 0} A_i\big)\,\, \mathbb{P}\text{-a.s.}, &  A_i \cap A_j = \emptyset, \quad \text{ for } i \ne j, \\
J(.;\nu^{i,\epsilon}) \leq \phi + 3\epsilon, & \text{ on } A_i, \text{ for } i \geq 1,
\end{array}
\right.
\numberthis \label{Eq:theo_1_eq5}
\end{align*} 
with $\nu^{i,\epsilon} = \nu^{(t_i,y_i),\epsilon}$ for any $i \geq 1$. We are now ready to prove \eqref{Eq:Theo_DPP1_2}. To this end, 
set $A^n = \cup_{0 \leq i \leq n} A_i$ for any $n \geq 1$ and define 
\begin{equation*}
\nu^{\epsilon,n,2}_s = \mathbf{1}_{[t,\theta]}(s) \nu^2_s + \mathbf{1}_{(\theta,T]} \big( \nu^2_s \mathbf{1}_{(A^n)^c}\big(\theta,X^\nu_{t,x}(\theta)\big) + \sum_{i=1}^n \mathbf{1}_{A_i}\big(\theta,X^\nu_{t,x}(\theta)\big)  \nu^{i,\epsilon,2}_s \big), \quad \forall s \in [t,T],
\end{equation*}
and $\bar{\nu}^{\epsilon,n} = (\nu^1, \nu^{\epsilon,n,2})$ for any $i \geq 1$. Note that $\{ (\theta,X^{\nu}_{t,x}(\theta)) \in A_i\} \in \mathcal{F}^t_\theta$ as a consequence of Assumption \ref{Assump:A}.1. Then, it follows from Assumption \ref{Assump:A}.3 that $\bar{\nu}^{\epsilon,n}\in \mathcal{U}_t$.
Moreover, by definition of $B(t_i,y_i;r_i)$, we have $ \theta = \bar{t}_i \leq t_i $ on $\{ (\theta,X^{\nu}_{t,x}(\theta) \in A_i\}$. Then, using Assumption \ref{Assump:A}.4, Assumption \ref{Assump:A}.2, and \eqref{Eq:theo_1_eq5}, we deduce 
\begin{align*}
&\Esp[f\big(X^{\bar{\nu}^{\epsilon,n}}_{t,x}(T)\big)|\mathcal{F}_\theta ]\mathbf{1}_{A^n}\big(\theta,X^\nu_{t,x}(\theta)\big) \\
= &  \Esp[ f\big(X^{\bar{\nu}^{\epsilon,n}}_{t,x}(T)\big)|\mathcal{F}_\theta ] \mathbf{1}_{A_0} \big(\theta,X^\nu_{t,x}(\theta)\big) + \sum_{i=1}^n \Esp[ f\big(X^{\bar{\nu}^{\epsilon,n}}_{t,x}(T)\big)|\mathcal{F}_{\bar{t}_i}  ] \mathbf{1}_{A_i} \big(\theta,X^\nu_{t,x}(\theta)\big) \\
	   =& V\big(T,X^{\bar{\nu}^{\epsilon,n}}_{t,x}(T)\big) \mathbf{1}_{A_0} \big(\theta,X^\nu_{t,x}(\theta)\big) + \sum_{i=1}^n J(\bar{t}_i ,X^{\nu}_{t,x}(\bar{t}_i); \nu^{i,\epsilon}) \mathbf{1}_{A_i} \big(\theta,X^\nu_{t,x}(\theta)\big) \\
	   \leq &\sum_{i=0}^n \big( \phi(\theta,X^\nu_{t,x}(\theta)) + 3\epsilon \big) \mathbf{1}_{A_i} \big(\theta,X^\nu_{t,x}(\theta)\big) = \big( \phi(\theta,X^\nu_{t,x}(\theta)) + 3\epsilon \big) \mathbf{1}_{A^n} \big(\theta,X^\nu_{t,x}(\theta)\big),
\end{align*}
Using the tower property of conditional expectations, we get 
\begin{align*}
\Esp\big[ f\big(X^{\bar{\nu}^{\epsilon,n}}_{t,x}(T)\big) \big] & = \Esp\big[\Esp[ f\big(X^{\bar{\nu}^{\epsilon,n}}_{t,x}(T)\big) |  \mathcal{F}_{\theta}] \big]\\
	    & \leq \Esp\big[\big( \phi(\theta,X^\nu_{t,x}(\theta)) + 3\epsilon \big) \mathbf{1}_{A^n} \big(\theta,X^\nu_{t,x}(\theta)\big)\big] + \Esp\big[f\big(X^{\bar{\nu}^{\epsilon,n}}_{t,x}(T)\big)\mathbf{1}_{(A^n)^c} \big(\theta,X^\nu_{t,x}(\theta)\big)\big].
\end{align*}
Since $f\big(X^{\bar{\nu}^{\epsilon,n}}_{t,x}(T) \in \mathbb{L}^1$, it follows from the dominated convergence theorem  
\begin{align*}
& \Esp\big[ f\big(X^{\bar{\nu}^{\epsilon,n}}_{t,x}(T)\big) \big] \\
  \leq & 3\epsilon + \underset{n \rightarrow \infty}{\lim \inf}\,\Esp\big[\phi(\theta,X^\nu_{t,x}(\theta))  \mathbf{1}_{A^n} \big(\theta,X^\nu_{t,x}(\theta)\big)\big]\\
        = & 3\epsilon + \underset{n \rightarrow \infty}{\lim}\,\Esp\big[\phi(\theta,X^\nu_{t,x}(\theta))^+ \mathbf{1}_{A^n} \big(\theta,X^\nu_{t,x}(\theta)\big)\big] -  \underset{n \rightarrow \infty}{\lim}\,\Esp\big[\phi(\theta,X^\nu_{t,x}(\theta))^- \mathbf{1}_{A^n} \big(\theta,X^\nu_{t,x}(\theta)\big)\big]\\
       = & 3\epsilon + \Esp\big[\phi(\theta,X^\nu_{t,x}(\theta))\big], \numberthis \label{Eq:theo_1_eq6}
\end{align*}
where the last inequality follows from the left-hand side of \eqref{Eq:theo_1_eq5} and from the monotone convergence theorem since either $\Esp\big[\phi(\theta,X^\nu_{t,x}(\theta))^+ \big] < \infty $ or $ \Esp\big[\phi(\theta,X^\nu_{t,x}(\theta))^-\big] < \infty$. Finally, by definition of $V$, definition of $\bar{V}$, the arbitrariness of $\nu^1$, and \eqref{Eq:theo_1_eq6}, we deduce
\begin{align*}
V(t,x) \leq \bar{V}(t,x; \nu^{\epsilon,n,2}) & = \sup_{\nu^1 \in \mathcal{U}^1_t} \Esp\big[f\big(X^{\bar{\nu}^{\epsilon,n}}_{t,x}(T)\big) \big]  \leq 3 \epsilon + \sup_{\nu^1 \in \mathcal{U}^1_t} \Esp\big[\phi(\theta,X^\nu_{t,x}(\theta))\big],
\end{align*}
which completes the proof of \eqref{Eq:Theo_DPP1_2} by the arbitrariness of $\nu^2 \in \mathcal{U}^2_t$ and $\epsilon > 0$.
\end{proof}

\subsection{Stochastic Differential Games under Controlled Markov Diffusions}
\manuallabel{subsec:mkv_diff_ctrl_pbm}{5.5}
Now considering a particular class of controlled Markov dynamics, where for any control process $\nu \in \mathcal{U}_0$, 
\begin{align*}
d X^\nu_t = b(t,X^\nu_t,\nu_t) dt + \tilde{\sigma}(t,X^\nu_t,\nu_t) d W_t, \numberthis \label{Eq:dynamics_x_u}
\end{align*}
with $W$ a $d$-dimensional Brownian motion, $b: \mathbb{R}_+ \times \mathbb{R}^d \times \mathbb{R}^{k^1} \times \mathbb{R}^{k^2} \rightarrow \mathbb{R}^d$ and $\tilde{\sigma}: \mathbb{R}_+ \times \mathbb{R}^d \times \mathbb{R}^{k^1} \times \mathbb{R}^{k^2} \rightarrow \mathcal{M}_{\mathbb{R}}(d)$ two continuous functions, and $\mathcal{M}_{\mathbb{R}}(d)$ the space of $d\times d$ matrices with real coefficients. Here we assume that $b$ and $\tilde{\sigma}$ satisfy the usual Lipschitz continuity conditions to ensure that Equation \eqref{Eq:dynamics_x_u} admits a unique strong solution $X^\nu_{t_0,x_0} $ such that $X^\nu_{t_0,x_0}(t_0) = x_0$ with $(t_0,x_0)\in \mathbb{S}$.
Moreover, the function $f:\mathbb{R}^d \rightarrow \mathbb{R}$ associated with the reward function $J$ in \eqref{Eq:reward_fct} is continuous, and there exists a constant $K$ such that 

$$
|f(x)|  \leq K \big(1 + \|x\|^2 \big), \qquad  \forall x \in \mathbb{R}^d.
$$

Then we can characterize the value of the game after defining the operator $H$ as follows: 
\begin{equation*}
H(t,x,p,A)  = \max_{u^1 \in \mathbb{U}^1} \min_{u^2 \in \mathbb{U}^2} H^{(u^1,u^2)}(t,x,p,A), \qquad \forall (t,x,p,q) \in \mathbb{S} \times \mathbb{R}^d \times \mathcal{M}_{\mathbb{R}}(d),
\end{equation*}
with $\mathbb{U}^1$ (resp. $\mathbb{U}^2$) a closed subset of $\mathbb{R}^{k^1}$ (resp. $\mathbb{R}^{k^2}$) and 
\begin{equation*}
 H^{(u^1,u^2)}(t,x,p,A) =  -b^\top p - \cfrac{1}{2} \Tr[\tilde{\sigma} \tilde{\sigma}^\top A],  \qquad \forall (t,x,p,q) \in \mathbb{S} \times \mathbb{R}^d \times \mathcal{M}_{\mathbb{R}}(d).   \end{equation*}
 
\begin{prop} Assume that $V$ is locally bounded. Then
\begin{itemize}
\item  $V_*$ is a viscosity supersolution of 
\begin{equation}
-{V_*}_t + H(.,{V_*},{V_*}_x,{V_*}_{xx}) \geq 0, \quad \text{ on } [0,T) \times \mathbb{R}^d,
\label{Eq:prop0_supersol_visco}
\end{equation} with
$$
V_*(t,x) = \underset{(t',x') \rightarrow (t,x)}{\liminf} V(t,x).
$$
\item $V^*$ is a viscosity subsolution of 
\begin{equation}
-V^*_t + H(.,V^*,V^*_x,V^*_{xx}) \leq 0, \quad \text{ on } [0,T) \times \mathbb{R}^d,\label{Eq:prop0_subsol_visco}
\end{equation}
with $$ V^*(t,x)= \underset{(t',x') \rightarrow (t,x)}{\limsup} V(t,x).$$

\item $V$ is the unique viscosity solution of 
\begin{equation}
-V_t + H(.,V,V_x,V_{xx}) = 0, \quad \text{ on } [0,T) \times \mathbb{R}^d.\label{Eq:prop0_sol_visco_0}
\end{equation}
\end{itemize}
\label{prop0:Visco_sol}
\end{prop}

\begin{proof}
{(For Proposition \ref{prop0:Visco_sol}).} Let us first prove the supersolution property \eqref{Eq:prop0_supersol_visco}.
\begin{enumerate} 
\item  For this, assume to the contrary that there is $(t_0,x_0)\in \mathbb{S}=[0,T] \times \mathbb{R}^d$ together with a smooth function $\phi:\mathbb{S} \rightarrow \mathbb{R}$ satisfying 
\begin{equation*}
0 = (V_*-\phi)(t_0,x_0) < (V_*-\phi)(t,x), \quad \forall (t,x) \in [0,T] \times \mathbb{R}^d,\quad (t,x) \ne (t_0,x_0),
\end{equation*} 
such that 
\begin{equation*}
\big(- \partial_t \phi  + H(.,\phi,\phi_x,\phi_{xx} ) \big) (t_0,x_0) < 0.
\end{equation*}
For $\tilde{\epsilon}>0$, define $\psi$ by 
\begin{equation*}
\psi(t,x) = \phi(t,x) - \tilde{\epsilon} \big( |t-t_0|^2 + \|x - x_0\|^4 \big),
\end{equation*}
and note that $\psi$ converges uniformly on compact sets to $\phi$ as $\tilde{\epsilon} \rightarrow 0$. Since $H$ is upper-semicontinuous and $(\psi,\psi_t,\psi_x,\psi_{xx})(t_0,x_0) = (\phi,\phi_t,\phi_x,\phi_{xx})(t_0,x_0)$, we can choose $\tilde{\epsilon} > 0$ small enough so that there exist $r > 0$, with $t_0 + r < T$, such that for any $u^{1} \in \mathbb{U}^1$ one can find some $\bar{u}^2 \in  \mathbb{U}^2$ satisfying
\begin{equation}
\big(- \partial_t \psi  + H^{(u^1,\bar{u}^2)}(.,\psi,\psi_x,\psi_{xx} ) \big) (t_0,x_0) < 0, \qquad \forall (t,x) \in B_r(t_0,x_0),
\label{Eq:prop0_eq1}
\end{equation}
with $B_r(t_0,x_0) $ the open ball of radius $r$ and center $(t_0,x_0)$. Let $(t_n,x_n)_{n\geq 1}$ be a sequence in $B_r(t_0,x_0)$ such that $\big(t_n, x_n,V(t_n,x_n)\big) \rightarrow \big(t_0, x_0,V_*(t_0,x_0)\big) $, $\nu^1 \in \mathcal{U}^1_{t_n}$, $\nu^2$ be the constant control $\nu^2 =  \bar{u}^{2}$, and $\nu = (\nu^1, \nu^2)$. Now write $X^n_. = X^{\nu}_{t_n,x_n}(.) $ for the solution of \eqref{Eq:dynamics_x_u} with control $\nu$ and initial condition $X^n_{t_n} = x_n$, and consider the stopping time 
\begin{equation*}
\theta_n = \inf\{ s>t_n;\, (s,X^n_s) \not\in B_r(t_0,x_0)  \}.
\end{equation*}
Note that $\theta_n < T$ since $t_0 + r < T$. Using \eqref{Eq:prop0_eq1} gives 
\begin{equation*}
\Esp[\int_{t_n}^{\theta_n}[- \partial_t \psi  + H^u(,\psi,\psi_x,\psi_{xx} ) ](s,X^{n}_{s})\, ds] < 0.
\end{equation*}
By the arbitrariness of $\nu^1$, we get 
\begin{equation}
\inf_{\nu^{2} \in \mathcal{U}^2_{t_n} } \sup_{\nu^1 \in \mathcal{U}^1_{t_n}} \, \Esp[\int_{t_n}^{\theta_n}[- \partial_t\psi  + H^u(,\psi,\psi_x,\psi_{xx} ) ](s,X^{n}_{s})\, ds] < 0.
\label{Eq:prop0_eq1_1}
\end{equation}
Applying It\^{o}'s formula to $\psi$ and using \eqref{Eq:prop0_eq1_1}, we deduce that
\begin{align*}
\psi(t_n,x_n)  < \inf_{\nu^{2} \in \mathcal{U}^2_{t_n} } \sup_{\nu^1 \in \mathcal{U}^1_{t_n}} \, \Esp[ \psi(\theta_n,X^{n}_{\theta_n})].
\end{align*}
Now observe that $\phi > \psi + \eta$ on $\big([0,T]\times \mathbb{R}^d\big) \setminus B_r(t_0, x_0)$ for some $\eta >0$. Hence, the above inequality implies that $\psi(t_n,x_n) < \inf_{\nu^{2} \in \mathcal{U}^2_{t_n} } \sup_{\nu^1 \in \mathcal{U}^1_{t_n}} \Esp[ \phi(\theta_n,X^{n}_{\theta_n})] - \eta$. Since $(\psi -V)(t_n, x_n) \rightarrow 0$, we can then find $n$ large enough so that 
\begin{equation*}
V(t_n,x_n) < \inf_{\nu^{2} \in \mathcal{U}^2_{t_n} } \sup_{\nu^1 \in \mathcal{U}^1_{t_n}}\Esp[ \phi(\theta_n,X^{n}_{\theta_n})] - \eta/2.
\end{equation*}
Meanwhile, Theorem \ref{Theo:DPP1} ensures that
\begin{equation*}
V(t_n,x_n) \geq \inf_{\nu^{2} \in \mathcal{U}^2_{t_n} } \sup_{\nu^1 \in \mathcal{U}^1_{t_n}} \, \Esp[ \phi(\theta_n,X^{\nu}_{t_n,X_n}(\theta_n))].
\end{equation*}
which gives the required contradiction.

\item We now move to the proof of \eqref{Eq:prop0_subsol_visco}, which is similar to that of \eqref{Eq:prop0_supersol_visco}. We assume to the contrary that there is $(t_0,x_0)\in \mathbb{S}$ together with a smooth function $\phi:\mathbb{S} \rightarrow \mathbb{R}$ satisfying 
\begin{equation*}
0 = (V^*-\phi)(t_0,x_0) > (V^*-\phi)(t,x), \quad \forall (t,x) \in [0,T] \times \mathbb{R}^d,\quad (t,x) \ne (t_0,x_0),
\end{equation*} 
such that 
\begin{equation*}
\big(- \partial_t \phi  + H(.,\phi,\phi_x,\phi_{xx} ) \big) (t_0,x_0) > 0.
\end{equation*}
For $\tilde{\epsilon}>0$, define $\psi$ by 
\begin{equation*}
\psi(t,x) = \phi(t,x) + \tilde{\epsilon} \big( |t-t_0|^2 + \|x - x_0\|^4 \big),
\end{equation*}
and note that $\psi$ converges uniformly on compact sets to $\phi$ as $\tilde{\epsilon} \rightarrow 0$. Since $H$ is lower-semicontinuous and $(\psi,\psi_t,\psi_x,\psi_{xx})(t_0,x_0) ) = (\phi,\phi_t,\phi_x,\phi_{xx})(t_0,x_0)$, we can choose $\tilde{\epsilon} > 0$ small enough so that there exist $r > 0$, with $t_0 + r < T$, such that for any $u^{1} \in \mathbb{U}^1$ one can find some $\bar{u}^2 \in  \mathbb{U}^2$ satisfying
\begin{equation}
\big(- \partial_t \psi  + H^{(u^1,\bar{u}^2)}(.,\psi,\psi_x,\psi_{xx} ) \big) (t_0,x_0) > 0, \qquad \forall (t,x) \in B_r(t_0,x_0),
\label{Eq:prop0_eq1_11}
\end{equation}
with $B_r(t_0,x_0) $ the open ball of radius $r$ and center $(t_0,x_0)$. Let $(t_n,x_n)_{n\geq 1}$ be a sequence in $B_r(t_0,x_0)$ such that $\big(t_n, x_n,V(t_n,x_n)\big) \rightarrow \big(t_0, x_0,V^*(t_0,x_0)\big) $, $\nu^1 \in \mathcal{U}^1_{t_n}$, $\nu^2$ be the constant control $\nu^2 =  \bar{u}^{2}$, and $\nu = (\nu^1, \nu^2)$. Now write $X^n_. = X^{\nu}_{t_n,x_n}(.) $ for the solution of \eqref{Eq:dynamics_x_u} with control $\nu$ and initial condition $X^n_{t_n} = x_n$, and consider the stopping time 
\begin{equation*}
\theta_n = \inf\{ s>t_n;\, (s,X^n_s) \not\in B_r(t_0,x_0)  \}.
\end{equation*}
Note that $\theta_n < T$ since $t_0 + r < T$. Using \eqref{Eq:prop0_eq1_11} gives 
\begin{equation*}
\Esp[\int_{t_n}^{\theta_n}[- \partial_t \psi  + H^u(,\psi,\psi_x,\psi_{xx} ) ](s,X^{n}_{s})\, ds] > 0.
\end{equation*}
By the arbitrariness of $\nu^1$, we get 
\begin{equation}
\inf_{\nu^{2} \in \mathcal{U}^2_{t_n} } \sup_{\nu^1 \in \mathcal{U}^1_{t_n}} \, \Esp[\int_{t_n}^{\theta_n}[- \partial_t\psi  + H^u(,\psi,\psi_x,\psi_{xx} ) ](s,X^{n}_{s})\, ds] > 0.
\label{Eq:prop0_eq1_12}
\end{equation}
Applying It\^{o}'s formula to $\psi$ and using \eqref{Eq:prop0_eq1_12}, we deduce that
\begin{align*}
\psi(t_n,x_n)  > \inf_{\nu^{2} \in \mathcal{U}^2_{t_n} } \sup_{\nu^1 \in \mathcal{U}^1_{t_n}} \, \Esp[ \psi(\theta_n,X^{n}_{\theta_n})].
\end{align*}
Now observe that $\psi > \phi + \eta$ on $\big([0,T]\times \mathbb{R}^d\big) \setminus B_r(t_0, x_0)$ for some $\eta >0$. Hence, the above inequality implies that $\psi(t_n,x_n) > \inf_{\nu^{2} \in \mathcal{U}^2_{t_n} } \sup_{\nu^1 \in \mathcal{U}^1_{t_n}} \Esp[ \phi(\theta_n,X^{n}_{\theta_n})] + \eta$. Since $(\psi -V)(t_n, x_n) \rightarrow 0$, we can then find $n$ large enough so that 
\begin{equation*}
V(t_n,x_n) > \inf_{\nu^{2} \in \mathcal{U}^2_{t_n} } \sup_{\nu^1 \in \mathcal{U}^1_{t_n}}\Esp[ \phi(\theta_n,X^{n}_{\theta_n})] + \eta/2.
\end{equation*}
Meanwhile, Theorem \ref{Theo:DPP1} ensures that
\begin{equation*}
V(t_n,x_n) \leq \inf_{\nu^{2} \in \mathcal{U}^2_{t_n} } \sup_{\nu^1 \in \mathcal{U}^1_{t_n}} \, \Esp[ \phi(\theta_n,X^{\nu}_{t_n,X_n}(\theta_n))],
\end{equation*}
which gives the required contradiction.
 
\item Since $V_*  \leq V \leq V^*$, the comparison principle in Lemma \ref{lem:comparison_principle_recall} below (see \citep{pham2009continuous}) shows that $V$ is the unique viscosity solution of \eqref{Eq:prop0_sol_visco_0} which completes the proof. 
\end{enumerate}
\end{proof}

\begin{lem}[Comparision principle]
Assume the same conditions as in Proposition \ref{prop0:Visco_sol}. Let $u$ and $v$ be respectively an upper-semicontinuous viscosity subsolution and a lower-semicontinuous viscosity supersolution of \eqref{Eq:prop0_sol_visco_0} such that $u(T,.) \leq v(T,.)$ on $\mathbb{R}^d$. Then, $u \leq v$ on $[0,T) \times \mathbb{R}^d$.
\label{lem:comparison_principle_recall}
\end{lem}

\section{Analysis of Optimal Adaptive Learning Rate and Batch Size}

This section will be devoted to analyzing the optimal learning rate and batch size but also discussing their implications for GANs training.

\subsection{Optimal Learning Rate}
\manuallabel{sec:OptLearningRate}{5.1}

Note that problem \eqref{Eq:ValFunctDef} is a special case of the more general framework in Section \ref{subsec:mkv_diff_ctrl_pbm} since the dynamic of the process $q$ in \eqref{Eq:sde} is a diffusion of the same form as Equation \eqref{Eq:dynamics_x_u}. Indeed, one can simply take the variables $X_t^\nu$, $\nu$, $b$, and $\tilde{\sigma}$ in \eqref{Eq:dynamics_x_u} as
$$
\left\{
\begin{array}{ccl}
X_t^\nu = q(t),   &  & b(t,x,v) = \left(
\begin{array}{c}
     v^w g_w(x) \\
     -v^\theta g_w(x) 
\end{array}
\right), \\
 & & \\
\nu = u, &   & \tilde{\sigma}(t,x,v)  = \left(
\begin{array}{cc}
     v^w \sqrt{\bar{\eta}^{w}} \sigma_w(x) & 0 \\
     0                            & v^\theta \sqrt{\bar{\eta}^{\theta}} \sigma_\theta(x)
\end{array}
\right),
\end{array}
\right.
$$
for any $x$ and $v = (v^w,v^\theta)$, and apply the general results of Section \ref{subsec:mkv_diff_ctrl_pbm} to analyze the optimal adaptive learning rate, assuming
\begin{Assumption}
\begin{enumerate}[label = D.\arabic*]
\item There exists a constant $L^1$ such that for $\phi = g_w,\,g_\theta,\sigma_w,\,\sigma_\theta$, we have 
\begin{equation*}
\begin{array}{ll}
\|\phi(w,\theta) - \phi(w',\theta')\|  \leq L^1 \big(\|w - w'\| + \|\theta - \theta'\|\big),&\\
\|\phi(w,\theta)\|\leq L^1 \big(1 + \|w\| + \|\theta\|\big),&
\end{array}
\end{equation*}
for any $(w,w',\theta,\theta')\in \big(\mathbb{R}^{M}\big)^2 \times \big(\mathbb{R}^{N}\big)^2$.
\item There exists a constant $K^1$ such that 
\begin{equation*}
|g(w,\theta)|  \leq K^1 \big(1 + \|w\|^2 + \|\theta\|^2 \big), \, \forall (w,\theta)\in \mathbb{R}^{M}\times \mathbb{R}^{N}.
\end{equation*}
\end{enumerate}
\label{Assump:D}
\end{Assumption}

In particular, it is clear that the value function $v$ is a solution to  the following Isaac-Bellman equation:
\begin{equation}
\left\{
\begin{array}{ll}
v_t + \max\min_{(u^w, u^\theta \in [u^{\min},u^{\max}])} & \big\{\big(u^w g_{w}^\top v_w - u^{\theta} g_{\theta}^\top v_\theta \big)\vspace{0.2cm}\\
& \hspace{-4cm}+ \frac{1}{2} \big[ (u^w)^2 (\bar{\Sigma}^w:v_{ww}) +    (u^\theta)^2 (\bar{\Sigma}^\theta:v_{\theta \theta})\big]\big\} = 0, \vspace{0.2cm}\\
v(T,\cdot) = g(\cdot),& 
\end{array}
\right.
\label{Eq:valfct_diffeq01}
\end{equation} 
with $A:B =  \Tr[A^\top B]$ for any real matrices $A$ and $B$. More precisely, we have

\begin{prop} 
Assume Assumption \ref{Assump:D}. Then 
\begin{itemize}
    \item  The value function $v$ defined in \eqref{Eq:ValFunctDef} is the unique viscosity solution of \eqref{Eq:valfct_diffeq01}.
    \item When $v \in \mathcal{C}^{1,2}([0,T],\mathbb{R}^M \times \mathbb{R}^N)$, the optimal learning rate $\bar{u}^w$ and $\bar{u}^\theta$ are given by
\begin{align*}
\hspace{0.5cm}
\bar{u}^w(t) = \left\{
\begin{array}{lll}
u^{\min} \vee \left( u^{w*}(t)\wedge u^{\max}\right), & \hspace{0.2cm} & \text{if } \, \left(\bar{\Sigma}^w:v_{ww}\right)\big(t,q(t)\big)  < 0,\\
\\
u^{\max}, & \hspace{0.2cm} & \text{if } \, |u^{\max}-u^{w*}(t)| \geq |u^{\min}-u^{w*}(t)|,\\
\\
u^{\min}, & \hspace{0.2cm} & \text{otherwise,}
\end{array} 
\right.
\end{align*}
and 
\begin{align*}
\bar{u}^\theta(t) = \left\{
\begin{array}{lll}
u^{\min} \vee \left( u^{\theta*}(t) \wedge u^{\max} \right), & \hspace{0.2cm} & \text{if }\, \left(\bar{\Sigma}^\theta:v_{\theta \theta}\right) \big(t,q(t)\big) > 0,\\
\\
u^{\max}, & \hspace{0.2cm} & \text{if } \, |u^{\max}-u^{\theta*}(t)| \geq |u^{\min}-u^{\theta*}(t)|,\\
\\
u^{\min}, & \hspace{0.2cm} & \text{otherwise,}
\end{array} 
\right.
\end{align*}
with $u^{w*}(t) = \left(\cfrac{-g_{w}^\top v_w}{\bar{\Sigma}^w:v_{ww}}\right)\big(t,q(t)\big)$, $ u^{\theta*}(t) = \left(\cfrac{g_{\theta}^\top v_\theta}{\bar{\Sigma}^\theta:v_{\theta \theta}} \right)\big(t,q(t)\big)$, and $\bar{\Sigma}^w$ and $\bar{\Sigma}^\theta$ as
\begin{equation*}
\left\{
\begin{array}{lll}
\bar{\Sigma}^w (t,q) = \{\bar{\sigma}^w_t (\bar{\sigma}^{w}_t)^{\top}\}(q), & \bar{\Sigma}^\theta(t,q) =  \{\bar{\sigma}^\theta_t (\bar{\sigma}^{\theta})^{\top}_t\}(q),&\\
\\
\bar{\sigma}^w_t (q)  = \sqrt{\bar{\eta}^w(t)} \sigma^w (q), & \bar{\sigma}^\theta_t (q)  = \sqrt{\bar{\eta}^\theta(t)} \sigma^\theta(q),&
\end{array}
\right.
\end{equation*}
for any $t \in \mathbb{R}_+$, and $q = (w,\theta)\in \mathbb{R}^{M}\times \mathbb{R}^{N}$.
\end{itemize}
\label{Prop:OptiLR}
\end{prop}

\begin{proof}{(For Proposition \ref{Prop:OptiLR}).}
Since Assumption \ref{Assump:D} holds, one can simply use Proposition \ref{prop0:Visco_sol} to show that $v$ is the unique viscosity solution of \eqref{Eq:valfct_diffeq01}. Moreover, the control $\bar{u}^w(t)$ maximizes the quadratic function $u : \rightarrow u \big( g_{w}^\top v_w \big) ( t,q(t) ) + \frac{1}{2} u^2 \big( \bar{\Sigma}^w:v_{ww}\big) ( t,q(t) )$, and similarly the control $\bar{u}_t^\theta $ minimizes $u : \rightarrow -u \big( g_{\theta}^\top v_\theta \big) ( t,q(t) ) + \frac{1}{2} u^2 \big(\bar{\Sigma}^\theta:v_{\theta \theta}\big) ( t,q(t) )$. Direct computation completes the proof.
\end{proof}

\paragraph{Learning rate and GANs training.}
\manuallabel{sec:LRAndDivergence}{3.2}

Now we can see  the explicit dependency of optimal learning rate on the convexity of the objective function, and its relation to Newton's algorithm.  Specifically, \begin{itemize}
\item Proposition \ref{Prop:OptiLR} provides a two-step scheme for the selection of the optimal learning rate
\begin{enumerate}[label =  Step \arabic*., leftmargin = 1.5cm]
\item Use the Isaac-Bellman equation to get $\bar{u} = (\bar{u}^w,\bar{u}^\theta)$
\item Given $\bar{u}$, apply the gradient algorithm with the optimal learning rate  $\bar{u} \bullet \bar{\eta}$.
\end{enumerate}

\item  To get the expression of the optimal adaptive learning rate in Proposition \ref{Prop:OptiLR}, we need some regularity conditions on the value function $v$ such as $v \in \mathcal{C}^{1,2}([0,T],\mathbb{R}^M \times \mathbb{R}^N)$.  Conditions for such a regularity can be found in 
\citep{pimentel2019regularity}. 
\item When the value function $v$ does not satisfy the regularity requirement $v \in \mathcal{C}^{1,2}([0,T],\mathbb{R}^M \times \mathbb{R}^N)$, it is standard to use discrete time approximations \citep{barles1991convergence,kushner2001numerical,wang2008maximal}. Note that these approximations are shown to converge towards the value function. 

\item  The introduction of the clipping parameter $u^{\max}$ is closely related to the convexity issue discussed for GANs in Section \ref{sec:GANsTrainConv}. When the convexity condition $\bar{\Sigma}^w:v_{ww} < 0$ is violated, the learning rate takes the maximum value $u^{\max}$ or minimum value $u^{\min}$ to escape as quickly as possible from this non-concave region. The clipping parameter $u^{\max}$ is also used to prevent explosion in GANs training. Conditions under which explosion occurs are detailed in Proposition \ref{Prop:LRDivergence}. 

\item  The control $(\bar{u}^w,\bar{u}^\theta)$ of Proposition \ref{Prop:OptiLR} is closely related to  the standard Newton algorithm. To see this, take $\bar{\eta} = (1,1)$, $M = N = 1$, $\sigma^w = g_w$, $ \sigma^\theta = g_\theta$, and replace the value function $v$ by the suboptimal choice $g$. For such a configuration of the parameters and with the convexity conditions 
$$ 
\hspace{-0.6cm}\big(\bar{\Sigma}^w:v_{ww}\big) =  |g_w|^2 g_{ww} < 0, \quad \big(\bar{\Sigma}^\theta:v_{\theta \theta}\big) = |g_\theta|^2 g_{\theta \theta}> 0,
$$ 

the controls $\bar{u}^w$ and $\bar{u}^\theta$ become
\begin{equation}
\hspace{-0.6cm}\bar{u}^w(t)  = u^{\min} \vee \cfrac{-1}{ \bar{\eta}^w(t) g_{ww}(q(t))} \wedge u^{\max}, \quad \bar{u}^\theta (t)  = u^{\min} \vee \cfrac{1}{ \bar{\eta}^\theta(t) g_{\theta \theta}(q(t))} \wedge u^{\max}.
\label{Eq:ConnecNewtAlgo}
\end{equation} 
In absence of the clipping parameters $u^{\max}$ and $u^{\min}$, the variables $\bar{u}^w$ and $\bar{u}^\theta$ are exactly the ones used for Newton's algorithm.
\end{itemize}

\paragraph{Learning rate and GANs convergence revisited.}
We can now further analyze the impact of the learning rate on the convergence of GANs, generalizing the example of Section \ref{sec:GANTrainParamTun} in which poor choices of the learning rate destroy the convergence.\\ 

Let $\epsilon > 0$, and $\tilde{u} = (\tilde{u}^w,\tilde{u}^\theta)$ be 
$$
\tilde{u}(t) = \left(\cfrac{-2 |g_{w}|^2 }{\bar{\Sigma}^w:g_{ww}},\, \cfrac{2|g_{\theta}|^2}{\bar{\Sigma}^\theta:g_{\theta \theta}}\right) \big(t,q(t)\big), \qquad \forall t \geq 0.
$$

We assume the existence of  $\gamma>0$ such that
$$
\gamma \leq -\left(\bar{\Sigma}^w:g_{ww}\right)(t,q), \qquad \gamma \leq \left(\bar{\Sigma}^\theta:g_{\theta \theta}\right)(t,q),
$$
for every $t\geq 0$, and $q = (w,\theta) \in \mathbb{R}^{M} \times \mathbb{R}^{N}$. Then,

\begin{prop}
For any control process $u = (u^{w},u^{\theta}) \in \mathcal{U}^w \times \mathcal{U}^\theta$ such that 
\begin{itemize}

\item $u^{w} \geq (\tilde{u}^w \vee 1) + \epsilon$, there exists $\tilde{\epsilon} > 0$ satisfying
\begin{align*}
J(T,0,q_0;u) &\leq -\tilde{\epsilon} \times T + \Esp\left[ \int_{0}^T \big\{ - u^{\theta} |g_{\theta}|^2 + \frac{1}{2}  (u^\theta)^2 (\bar{\Sigma}^\theta:g_{\theta \theta})\big\}\big(s,q(s)\big)\,ds|q_0\right]\\
    & = -\tilde{\epsilon} \times T + L^1(T,q_0;u),
\numberthis\label{Eq:expl1}
\end{align*}
for any $(T,q_0) \in \mathbb{R}_+\times \mathbb{R}^{M+N}$. 

\item $u^{\theta} \geq (\bar{u}^\theta \vee 1) + \epsilon$, there exists $\tilde{\epsilon} > 0$ such that 
\begin{align*}
J(T,0,q_0;u) &\geq  \tilde{\epsilon} \times T + \Esp\left[ \int_{0}^T \big\{ u^{w} |g_{w}|^2  + \frac{1}{2}  (u^w)^2 (\bar{\Sigma}^w:g_{ww})\big\}\big(s,q(s)\big)\,ds|q_0\right]\\
    &  = \tilde{\epsilon} \times T + L^2(T,q_0;u),
\numberthis \label{Eq:expl2}
\end{align*}
for any $(T,q_0) \in \mathbb{R}_+\times \mathbb{R}^{M+N}$.
\end{itemize}
\label{Prop:LRDivergence}
\end{prop}

Note that  inequality \eqref{Eq:expl1} shows that \begin{equation}
J(T,0,q_0;u) \underset{T \rightarrow \infty}{\rightarrow} - \infty,
\label{Eq:Expl_explicit1}
\end{equation}
for any $q_0 \in \mathbb{R}^{M+N}$ satisfying $\lim\sup_{T \rightarrow \infty} L^1(T,q_0;u) < +\infty$. Similarly, inequality \eqref{Eq:expl2} gives
\begin{equation}
J(T,0,q_0;u) \underset{t \rightarrow \infty}{\rightarrow} + \infty,
\label{Eq:Expl_explicit2}
\end{equation}
for every $q_0 \in \mathbb{R}^{M+N}$ such that $\lim\inf_{T \rightarrow \infty} L^2(T,q_0;u) > -\infty$. Thus, Equations \eqref{Eq:Expl_explicit1} and \eqref{Eq:Expl_explicit2} guarantee the explosion of the reward function without proper  choices of the learning rate.

\begin{proof}{(For Proposition \ref{Prop:LRDivergence}).}
Since the proofs of inequalities \eqref{Eq:expl1} and \eqref{Eq:expl2} are similar, we will only prove \eqref{Eq:expl1}. Let $(T,q_0) \in \mathbb{R}_+ \times \mathbb{R}^{M+N}$. By It\^{o}'s formula and the SDE \eqref{Eq:sde2} for the process $(q(t))_{t \geq 0}$, we get
\begin{align*}
\partial_T J(T,0,q_0;u) & = \Esp\big[\big\{\big( u^w |g_{w}|^2 - u^{\theta} |g_{\theta}|^2 \big) \\
                &  \hspace{-2.2cm} + \frac{1}{2} \big[ (u^w)^2 (\bar{\Sigma}^w:g_{ww}) +    (u^w)^2 (\bar{\Sigma}^\theta:g_{\theta \theta})\big]\,\big\} \big(T,q(T)\big)|q_0\big] \\
			      &\hspace{-2.25cm}  = \Esp \big[h^1\big(T,q(T);u^w(T)\big) \big] + \Esp \big[h^2 \big(T,q(T);u^\theta(T)\big) \big],
\end{align*}
with 
$$
h^1(t,q;u) = u |g_{w}|^2(q) + \frac{1}{2} (u)^2 (\bar{\Sigma}^w:g_{ww})(t,q),
$$
and
$$
h^2(t,q;u) = -u |g_{\theta}|^2(q) + \frac{1}{2} (u)^2 (\bar{\Sigma}^\theta:g_{\theta \theta})(t,q),
$$
for any $(t,q,u) \in \mathbb{R}_+ \times \mathbb{R}^{M+N} \times [0,u^{\max}]$. Using 
$$
\gamma \leq -(\bar{\Sigma}^w:g_{ww})(t,q), 
$$
for every $(t,q) \in \mathbb{R}_+ \times \mathbb{R}^{M+N}$, and $u^{w} \geq (\tilde{u}^w \vee 1) + \epsilon$ almost surely, we obtain\footnote{Note that the convexity condition $(\bar{\Sigma}^w:g_{ww}) \leq 0$ is in force.} 
\begin{align*}
\Esp\big[h^1\big(T,q(T);u^{w}(T)\big)\big] & = \Esp\big[ \big(a\, u^{w} ( u^{w} - \tilde{u}^w ) \big) (T) \big] \\
                        & \hspace{-1.7cm}\leq -(\gamma/2) \Esp[u^{w} (T)] \epsilon \leq - (\gamma/2) (1 + \epsilon) \epsilon = -\tilde{\epsilon},
\end{align*}
with $a(T) = \cfrac{(\bar{\Sigma}^w:g_{ww})(T,q(T))}{2}$.
Thus, 
\begin{align*}
\partial_T J(T,0,q_0;u)\leq -\tilde{\epsilon} + \Esp\big[ h^2\big( T, q(T);u^\theta(T) \big) \big],
\end{align*}
which ensures that 
\begin{align*}
J (T, 0, q_0;u) \leq -\tilde{\epsilon} T + L^{\theta},
\end{align*}
for all $T \geq 0$. 
\end{proof}

\subsection{Optimal Batch Size}
\manuallabel{sec:batch_size}{5.2}

The optimal batch size problem can be formulated and analyzed by following the same approach of Section \ref{sec:OptLearningRate}.

\begin{prop} 

Under Assumption \ref{Assump:D},

\begin{itemize}
    \item The value function $v$ solving the batch size problem \eqref{Eq:ValFunctDef23} is the unique viscosity solution of the following equation:

\begin{equation}
\left\{
\begin{array}{ll}
v_t + \min_{m^\theta \in [1,m^{\max}]} \left\{\cfrac{\big( g_{w}^\top v_w - g_{\theta}^\top v_\theta \big)}{ m^\theta} + \cfrac{\left[  (\bar{\Sigma}^w:v_{ww}) + (\bar{\Sigma}^\theta:v_{\theta \theta})\right] }{2 (m^\theta)^2} \right\} = 0, \\
v(T,\cdot) = g(\cdot),& 
\end{array}
\right.
\label{Eq:HJB_BS_prop}
\end{equation} 
where $m^\theta$ controls the generator batch size and $m^{\max}$ is an upper bound representing the maximum batch size. 
    \item When $v \in \mathcal{C}^{1,2}([0,T],\mathbb{R}^M \times \mathbb{R}^N)$, the optimal control $\bar{m}^\theta$ is given by 
    
\begin{align*}
\bar{m}^\theta(t) = \left\{
\begin{array}{ll}
1 \vee \left( m^{*}(t) \wedge m^{\max}\right), &  \text{if }\, \left( (\bar{\Sigma}^w:v_{ww}) + (\bar{\Sigma}^\theta:v_{\theta \theta}) \right) \big(t,q(t)\big) > 0, \\
\\
m^{\max}, & \text{if } \, \big| \big( m^{\max} \big)^{-1} - \big( m^{*}(t) \big)^{-1} \big| \geq \big| 1 - \big( m^{*}(t) \big)^{-1} \big|,\\
\\
1, & \text{otherwise,}\\
\end{array} 
\right.
\end{align*}

with $m^{*}(t) = \left( \cfrac{ (\bar{\Sigma}^w:v_{ww}) + (\bar{\Sigma}^\theta:v_{\theta \theta}) }{ g_{w}^\top v_w - g_{\theta}^\top v_\theta } \right) \big(t,q(t)\big)$.

\end{itemize}
\label{Prop:OptiBS}
\end{prop}

The proof of Proposition \ref{Prop:OptiBS} is omitted since it is very similar to the one of Proposition \ref{Prop:OptiLR}. 

\paragraph{Some remarks.} We observe that 
\begin{itemize}
    \item Proposition \ref{Prop:OptiBS}, in the same spirit of Proposition \ref{Prop:OptiLR}, suggests a two-step approach for the implementation of the optimal batch size. First, estimate $\bar{m}^\theta$ using Equation \eqref{Eq:HJB_BS_prop}. Second, apply the gradient algorithm with $\bar{m}^\theta$.
    \item The expression of the optimal batch size is given under the regularity condition $v \in \mathcal{C}^{1,2}([0,T],\mathbb{R}^M \times \mathbb{R}^N)$. It is possible to show that $v \in \mathcal{C}^{1,2}([0,T],\mathbb{R}^M \times \mathbb{R}^N)$ when $g$ and the volatility $\bar{\sigma}$ satisfy simple Lipschitz continuity conditions (see \citep{evans1983classical} for more details). Finally, when $v$ does not satisfy the regularity assumption $v \in \mathcal{C}^{1,2}([0,T],\mathbb{R}^M \times \mathbb{R}^N)$, it is standard to use discrete time approximations (see Section \ref{sec:OptLearningRate} for more details).
    \item The convexity condition $(\bar{\Sigma}^w:v_{ww}) + (\bar{\Sigma}^\theta:v_{\theta \theta})$ involved in the expression of $\bar{m}^\theta$ is the aggregation of the generator component $(\bar{\Sigma}^w:v_{ww})$ and the discriminator one $(\bar{\Sigma}^\theta:v_{\theta \theta})$. Moreover, the clipping parameters $1$ and $m^{\max}$ are used to prevent either gradient explosion or vanishing gradient.
\end{itemize}

\section{Numerical Experiment}
\label{sec:Exp}

In this section, we compare a well-known and established algorithm, the ADAM optimizer, with its adaptive learning rate counterpart which we call LADAM. The implementation of the adaptive learning rate is detailed in Section \ref{lrAlgoDesign}. We use three numerical examples to compare the convergence speed of these algorithms: generation of Gaussian distributions, generation of Student t-distributions, and financial time series generation.

\subsection{Optimizer Design}
\manuallabel{lrAlgoDesign}{6.1}

For computational efficiency, instead of directly applying  the two-step resolution scheme introduced in Section \ref{sec:OptLearningRate} for high dimensional Isaac-Bellman equation,  we exploit the connection between the optimal learning rate and Newton's algorithm by following the methodology illustrated in Equation \eqref{Eq:ConnecNewtAlgo}. The main idea is to approximate the function $v$, introduced in Equation \eqref{Eq:ValFunctDef}, by $g$.\\
 
Furthermore, we divide all the parameters into $M$ sets. For example, for neural networks, one may associate a unique set to each layer of the network. For each set $i$,  denote by $x^i_t$ the value at iteration $t \in \mathbb{N}$ of the set's parameters and write $g_{x^i}$ for the gradient of the loss $g$ with respect to the parameters of the set $i$. Then, we replace for each set $i$ the update rule 
\begin{equation*}
x^i_{t+1} = x^i_{t} - \bar{\eta}_t g_{x^i}(x^i_{t}) , 
\end{equation*}
with $\bar{\eta}_t$ a base reference learning rate that depends on the optimizer, by the following update rule: 
\begin{equation*}
x^i_{t+1} = x^i_t - u^i_t \big(\bar{\eta}_t g_{x^i}(x^i_{t}) \big), 
\end{equation*}
where the adjustment $u^i_t$, derived from Proposition \ref{Prop:OptiLR}, is defined as follows:
$$ u^i_t = u^{\min} \vee \cfrac{\|g_{x^i}(x^i_{t})\|^2}{\Tr[(\bar{g}_{x^i} \bar{g}_{x^i}^\top) g_{x^ix^i}](x^i_{t})} \wedge u^{\max}
$$  
if $\big(\Tr[( \bar{g}_{x^i} \bar{g}_{x^i}^\top ) g_{x^ix^i}]\big)(x^i_{t}) > 0$;
otherwise 
$u^i_t= u^{\max}$.
Here $\bar{g}_{x^i} = \sqrt{\bar{\eta}_t} g_{x^i}$ and $u^{\max}$ (resp. $u^{\min}$) is the maximum (resp. minimum) allowed adjustment. Note that the same adjustment $u^i_t$ is used for all parameters of the set $i$. The most expensive operation here is the computation of the term $\big(\Tr[(g_{x^i} g_{x^i}^\top) g_{x^ix^i}]\big)(x^i_{t})$ since we need to approximate the Hessian matrix $g_{x^ix^i}$.

\subsection{Vanilla GANs Revisited}

We use the vanilla GANs in Section \ref{sec:GAN_wellpos} to show the relevance of our adaptive learning rate and the influence of the batch size on GANs training.

\paragraph{Data.} The numerical samples of $X$ are drawn either from the Gaussian distribution $N(m,\sigma^2)$ with $(m,\sigma) = (3,1)$ or the translated Student t-distribution $a + T(n)$ where $a = 3$, and $T(n)$ is a standard Student t-distribution with $n$ degrees of freedom. The noise $Z$ always follows $N(0,1)$. These samples are then decomposed into two sets: a training set and a test set. In the experiment, one epoch refers to the number of gradient updates needed to pass the entire training dataset. \\

At the end of the training,  the discriminator accuracy is expected to be around $50\%$, meaning that the generator manages to produce samples that fool the discriminator who is unable to differentiate the original data from the fake ones.

\paragraph{Network architecture description.} We work here under the same setting of Section \ref{sec:GAN_wellpos}, with discriminator detailed in Equation \eqref{Eq:simple_model} while the generator is composed of the following two layers: the first layer is linear, and the second one is convolutional. Both layers use the ReLU activation function.

\subsubsection{Gaussian Distribution Generation}
\manuallabel{numres}{6.2.1}

Here the input samples $X$ are drawn from $N(3,1)$. We first compare the accuracy of the discriminator for two choices of the optimizers, namely the standard ADAM optimizer and the ADAM optimizer with adaptive learning rates (i.e., LADAM). Figure \ref{Fig:Plot_accuracy_discriminator_sgdlsgd} plots the accuracy of the discriminator when the base learning rate varies for ADAM and LADAM. Evidently, a bigger learning rate yields faster convergence for both optimizers; LADAM outperforms ADAM due to the introduction of the adaptive learning rate, and more importantly, LADAM performance is more robust with respect to the initial learning rate since it can adjust using the additional adaptive component.

\begin{figure}[h!]
    \center
	(a) Discriminator accuracy ADAM \hspace{1.5cm} (b) Discriminator accuracy LADAM \\
	\includegraphics[width=0.45\textwidth]{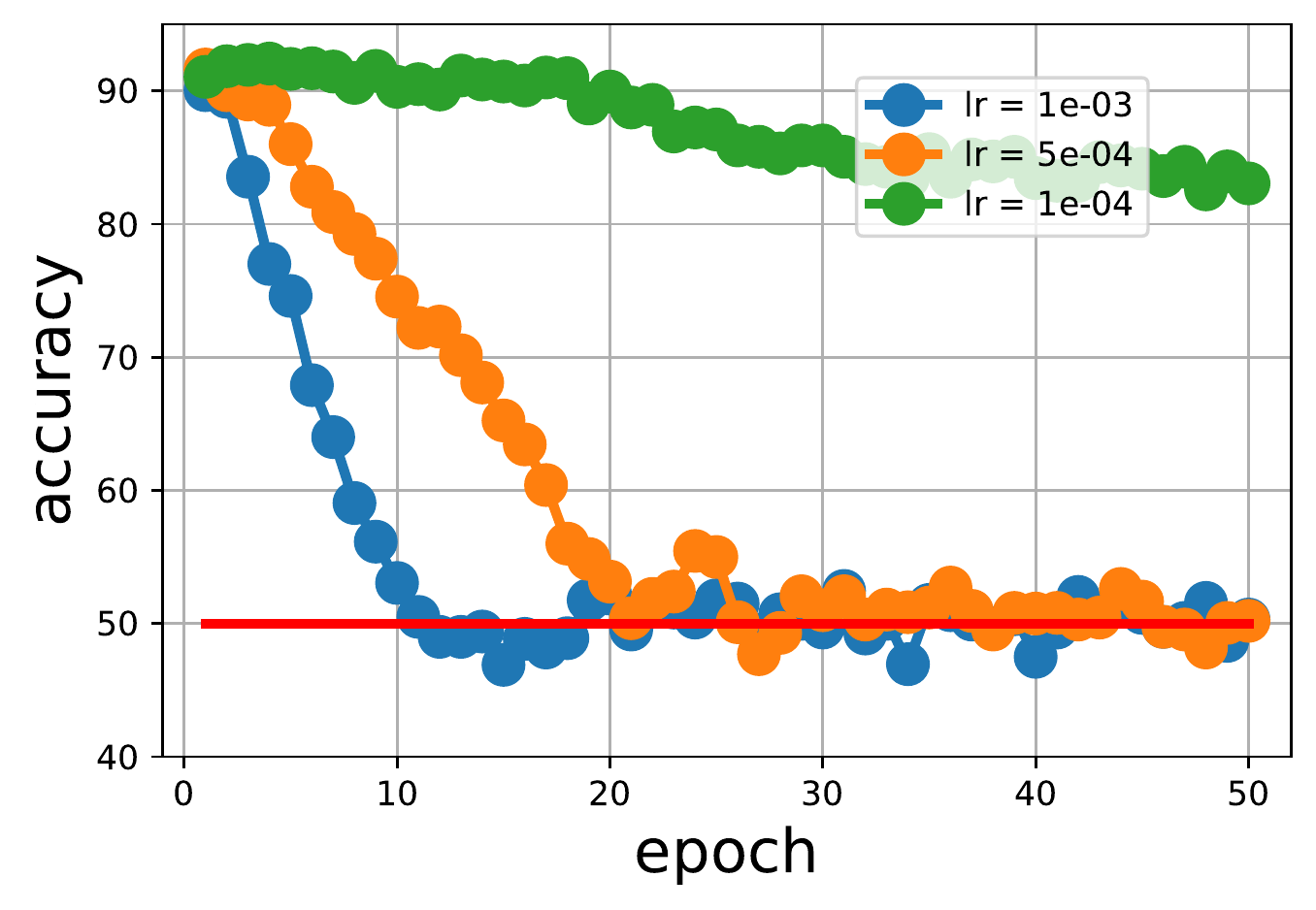}
	\includegraphics[width=0.45\textwidth]{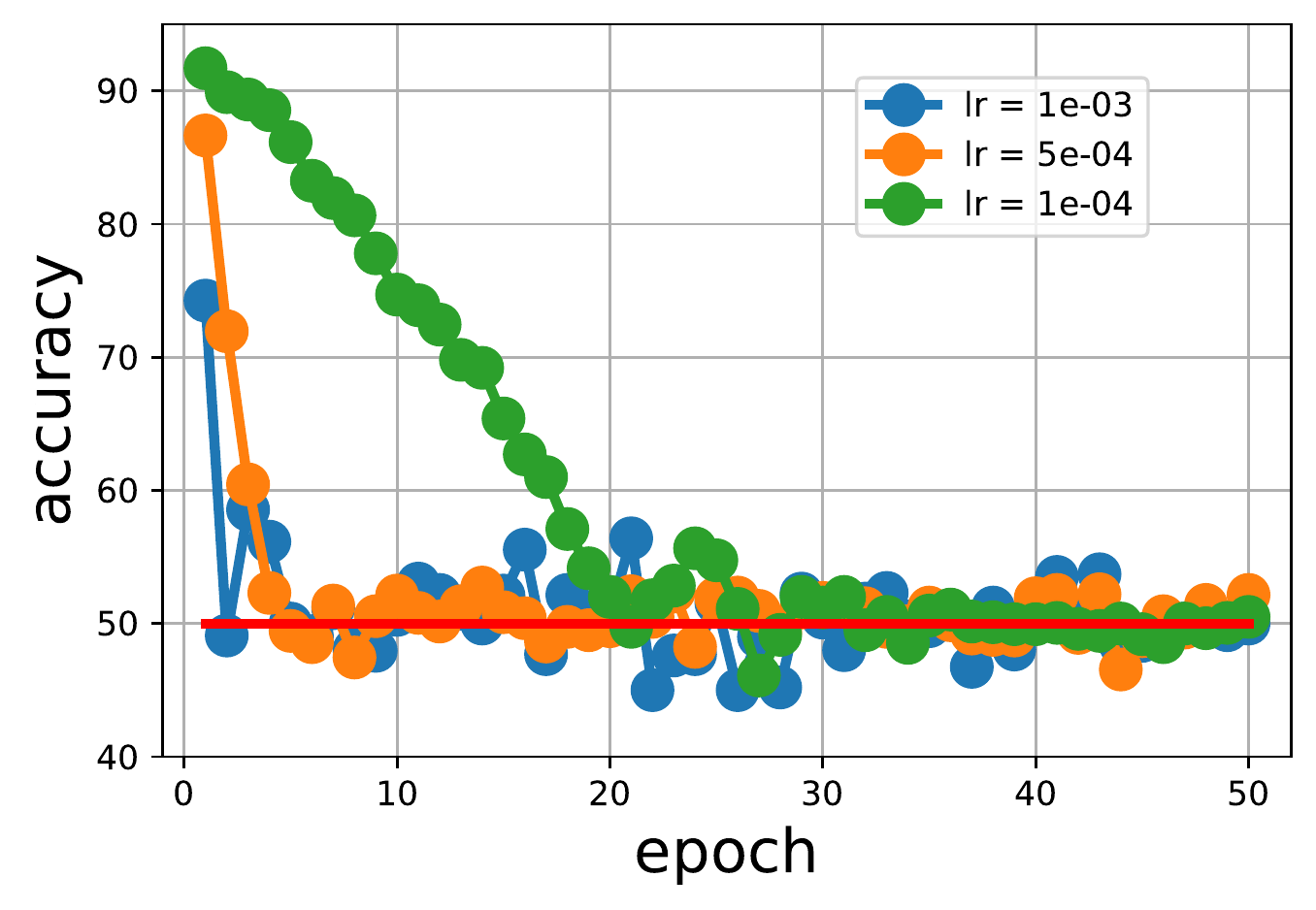}	
    \caption{Discriminator accuracy for ADAM with a base learning rate in (a) and LADAM with an adaptive learning rate in (b)}
    \label{Fig:Plot_accuracy_discriminator_sgdlsgd}
\end{figure}

Next, we analyze the generator loss for ADAM and LADAM optimizers. Figure \ref{Fig:Plot_loss_generator_sgdlsgd} shows the variations of the generator loss when moving the base learning rate for these optimizers. First, one can see that the loss decreases faster when using LADAM. Second, Figure \ref{Fig:Plot_loss_generator_sgdlsgd} confirms the robustness of LADAM with respect to the choice of the initial learning rate, again thanks to the additional adaptive learning rate component.
 
\begin{figure}[h!]
    \center
	 (a) Generator loss ADAM \hspace{1.5cm} (b) Generator loss LADAM \\
	\includegraphics[width=0.45\textwidth]{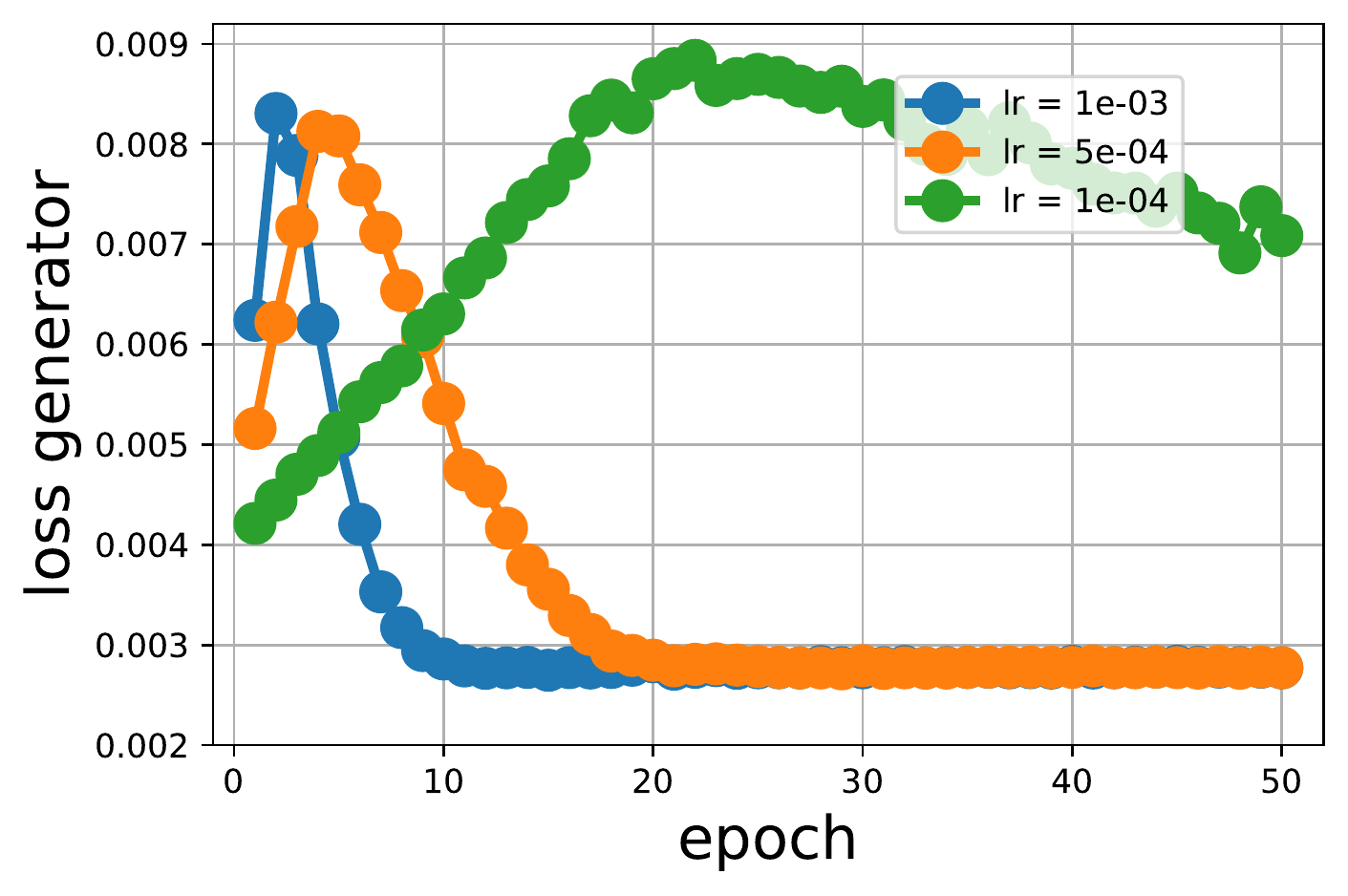}	\includegraphics[width=0.45\textwidth]{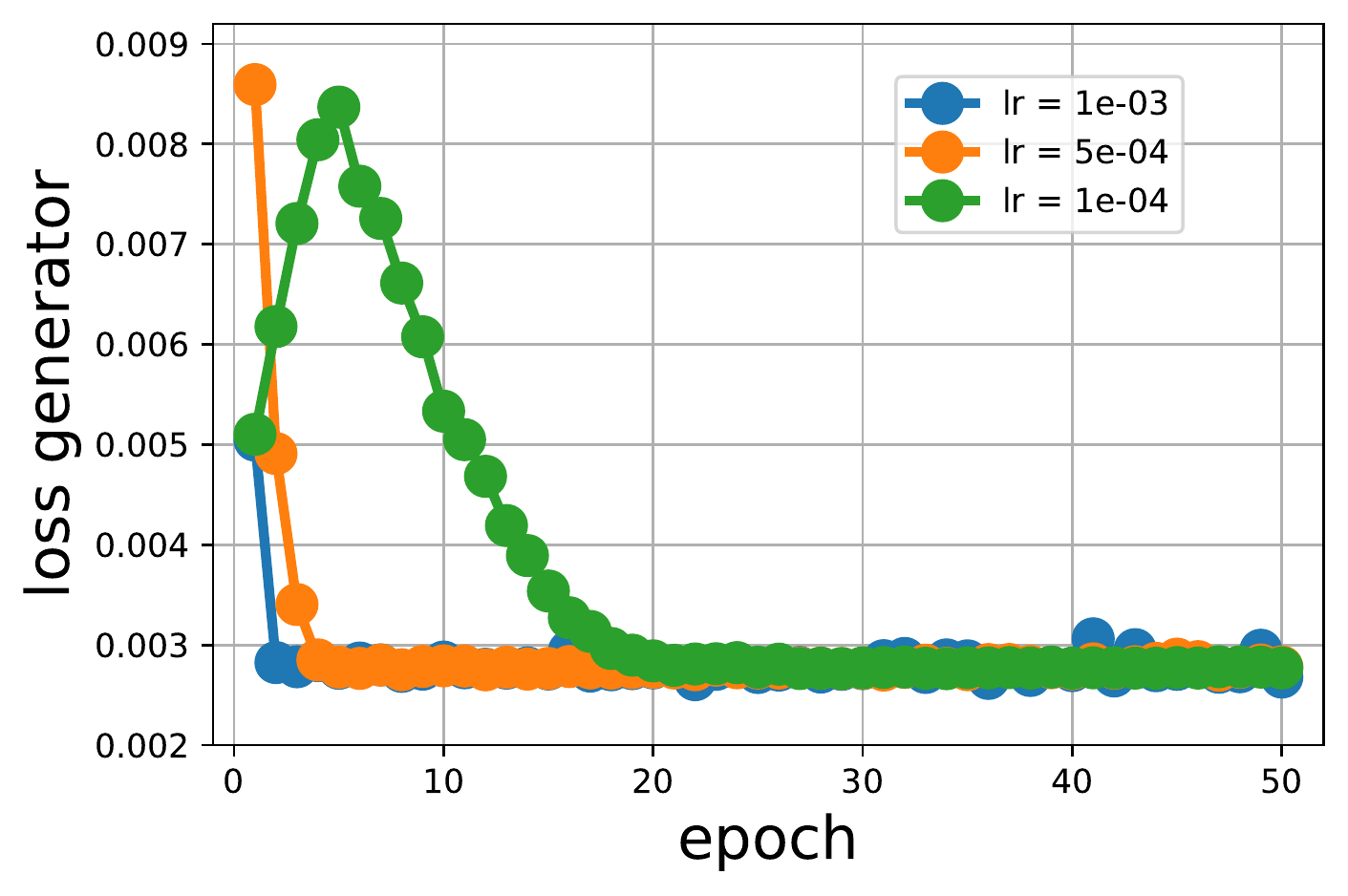}
    \caption{Generator loss for ADAM with a base learning rate in (a) and LADAM with an adaptive learning rate in (b).}
    \label{Fig:Plot_loss_generator_sgdlsgd}
\end{figure}

\subsubsection{Student T- Distribution Generation}

We repeat the same experiments of Section \ref{numres} but with different inputs data. Here, the input samples $X$ are drawn from the translated Student t-distribution $3 + T(3)$ where $T(3)$ is a standard Student t-distribution with $3$ degrees of freedom. Figures \ref{Fig:Plot_accuracy_discriminator_sgdlsgdStudent} and \ref{Fig:Plot_loss_generator_sgdlsgdStudent} display the discriminator accuracy and generator loss for both ADAM and LADAM. One can see that conclusions of Section \ref{numres} still hold, namely LADAM offers better performance, and is more robust with respect to the choice of the initial learning rate.

\begin{figure}[h!]
    \center
	(a) Discriminator accuracy ADAM \hspace{1.5cm} (b) Discriminator accuracy LADAM \\
	\includegraphics[width=0.45\textwidth]{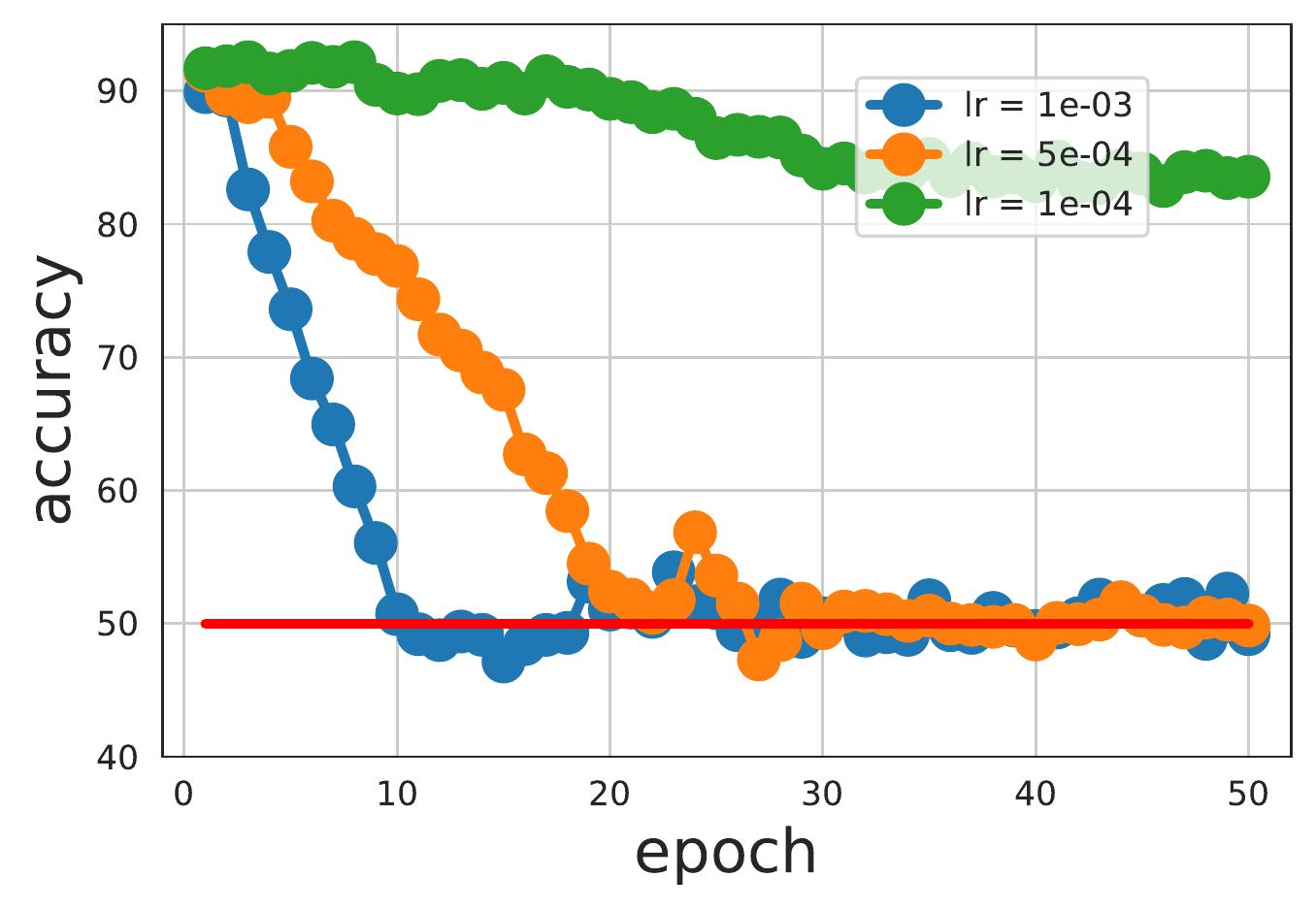}
	\includegraphics[width=0.45\textwidth]{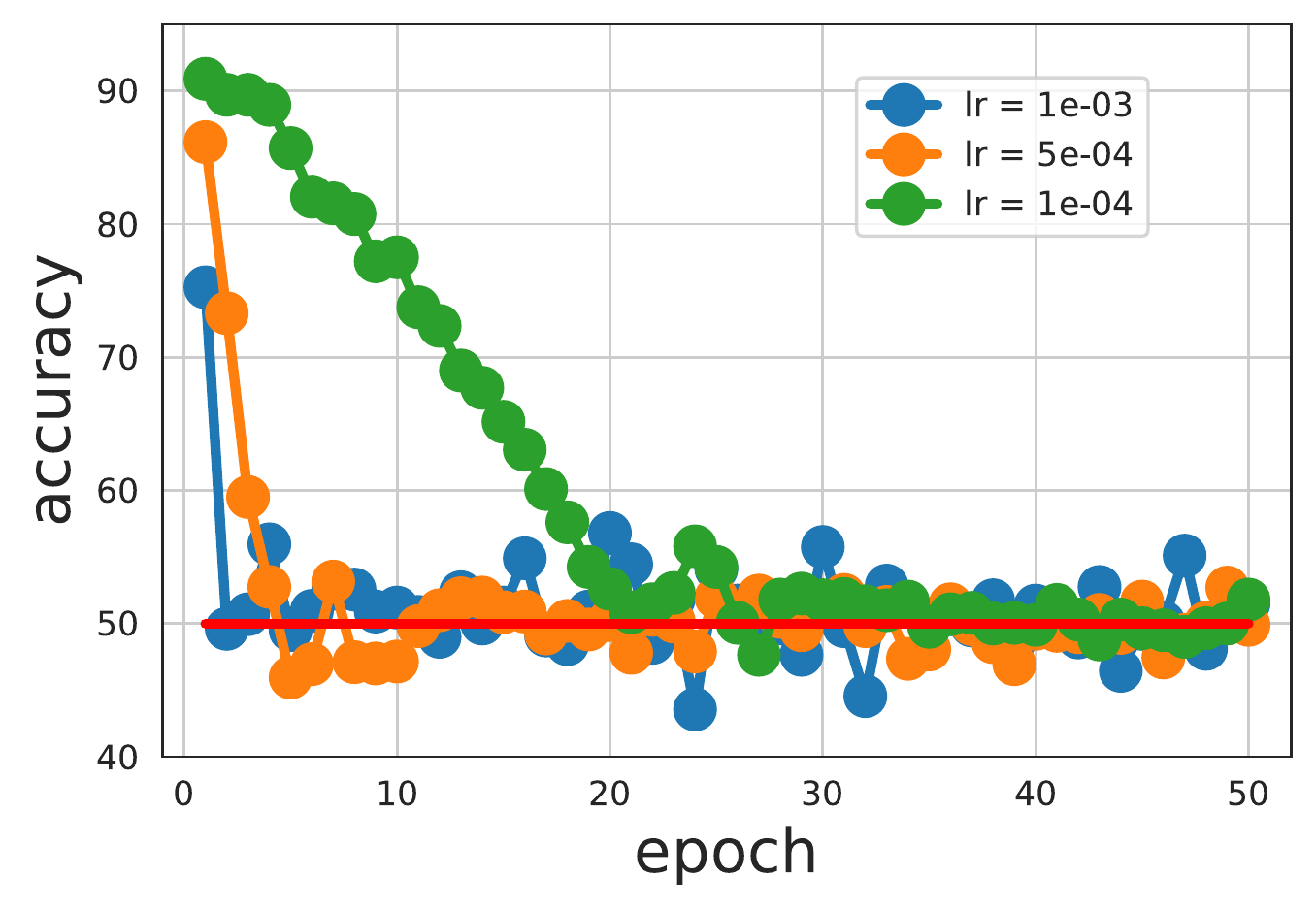}	
    \caption{Discriminator accuracy for ADAM with a base learning rate in (a) and LADAM with an additional adaptive learning rate component in (b).}
    \label{Fig:Plot_accuracy_discriminator_sgdlsgdStudent}
\end{figure}

\begin{figure}[h!]
    \center
	 (a) Generator loss ADAM \hspace{1.5cm} (b) Generator loss LADAM \\
	\includegraphics[width=0.45\textwidth]{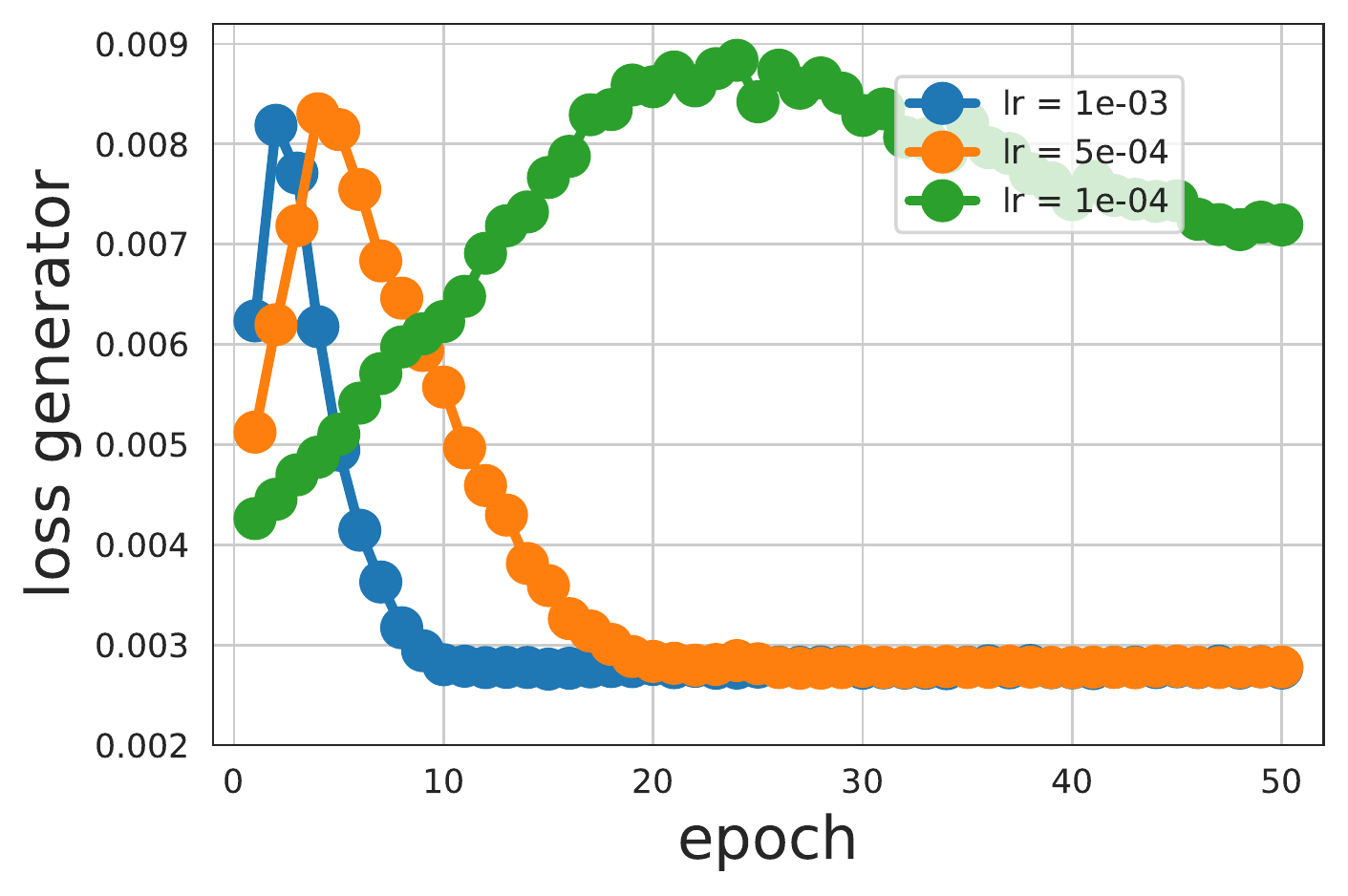}	\includegraphics[width=0.45\textwidth]{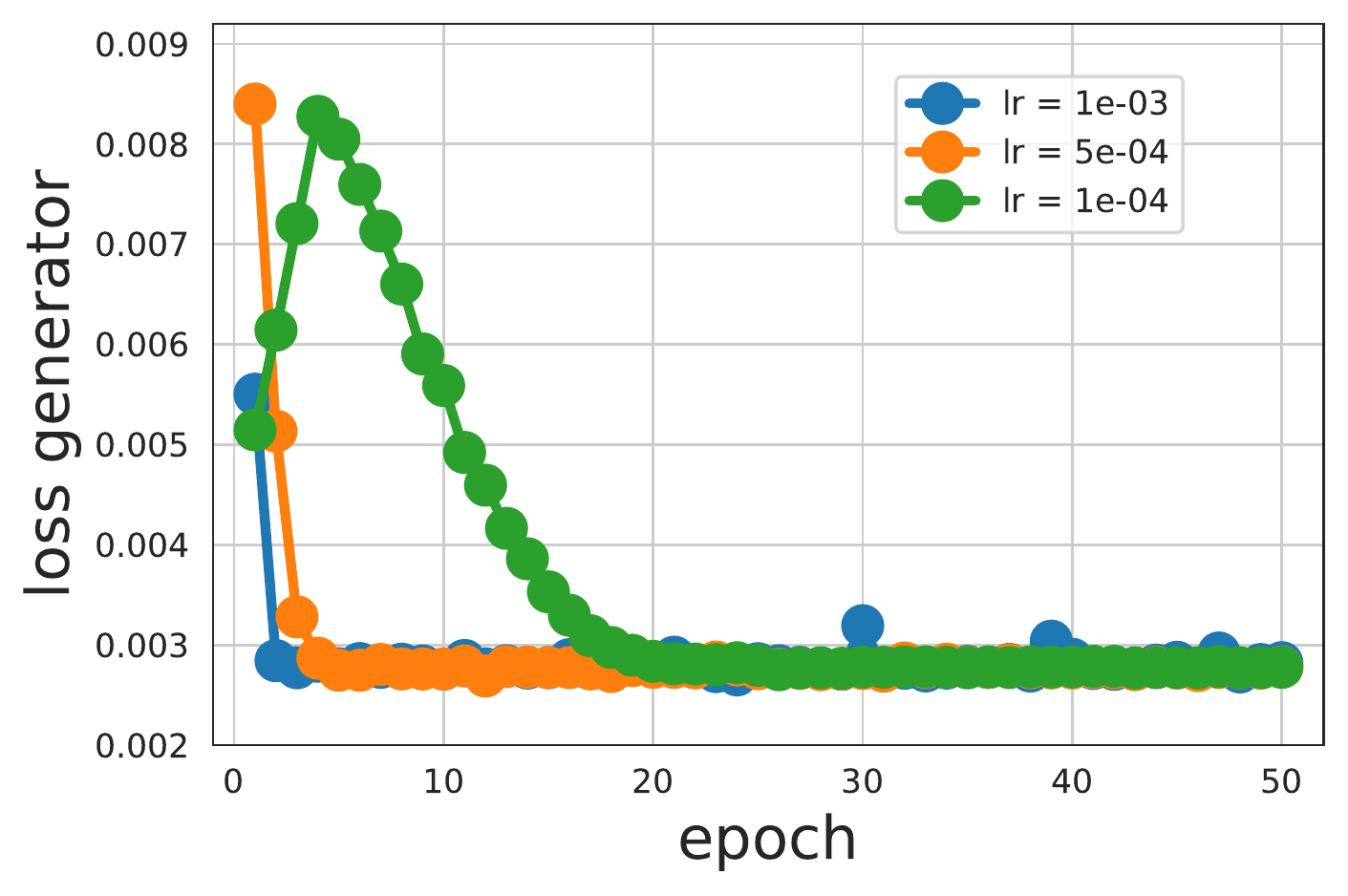}	
    \caption{Generator loss for ADAM with a base learning rate in (a) and LADAM with an additional adaptive learning rate component in (b).}
    \label{Fig:Plot_loss_generator_sgdlsgdStudent}
\end{figure}

\subsection{Financial Time Series Generation}

\paragraph{Data.} Data used here is taken from Quandl Wiki Prices database, which offers stock prices, dividends, and splits for 3000 US publicly-traded companies. In the numerical analysis, we will focus on the following six stocks: Boeing, Caterpillar, Walt Disney, General Electric, IBM, and Coca-Cola. For each stock, key quantities such as the traded volume, the open and close prices are daily recorded. The studied time period starts from January, 2000 and ends in March, 2018.

\paragraph{Network architecture description.} The neural network architecture is borrowed from \citep{yoon2019time}. It consists of four network components: an embedding function, a recovery function, a sequence generator, and a sequence discriminator. The key idea is to jointly train the auto-encoding components (i.e., the embedding and the recovery functions) with the adversarial components (i.e., the discriminator and the generator networks) in order to learn how to encode features and to generate data at the same time.

\paragraph{Numerical results.} Figure \ref{Fig:Plot_loss_generator_acc_discrim_TimeGan} plots the discriminator accuracy for the two following optimizers:  the standard ADAM  optimizer, and the same ADAM  optimizer with an adaptive learning rate, which we call LADAM. We fix the base learning rate here to be $5 \cdot 10^{-4}$ since the best performance for ADAM is obtained with this value.\\

Figure \ref{Fig:Plot_loss_generator_acc_discrim_TimeGan} clearly shows that both accuracy and loss of LADAM converge faster than those of ADAM. This is mainly due to the adaptive learning component which controls how fast the algorithm moves in the gradient direction. Moreover, Figure \ref{Fig:Plot_loss_generator_acc_discrim_TimeGan}.b reveals that the loss from LADAM is more stable (i.e., with less fluctuations) than that from ADAM. Such a behavior is expected since LADAM incorporates the convexity of the loss function in its choice of the learning rate.

\begin{figure}[h!]
    \center
	\hspace{-2.cm} (a) Discriminator accuracy  \hspace{1.5cm} (b) Generator loss \\
	\includegraphics[width=0.45\textwidth]{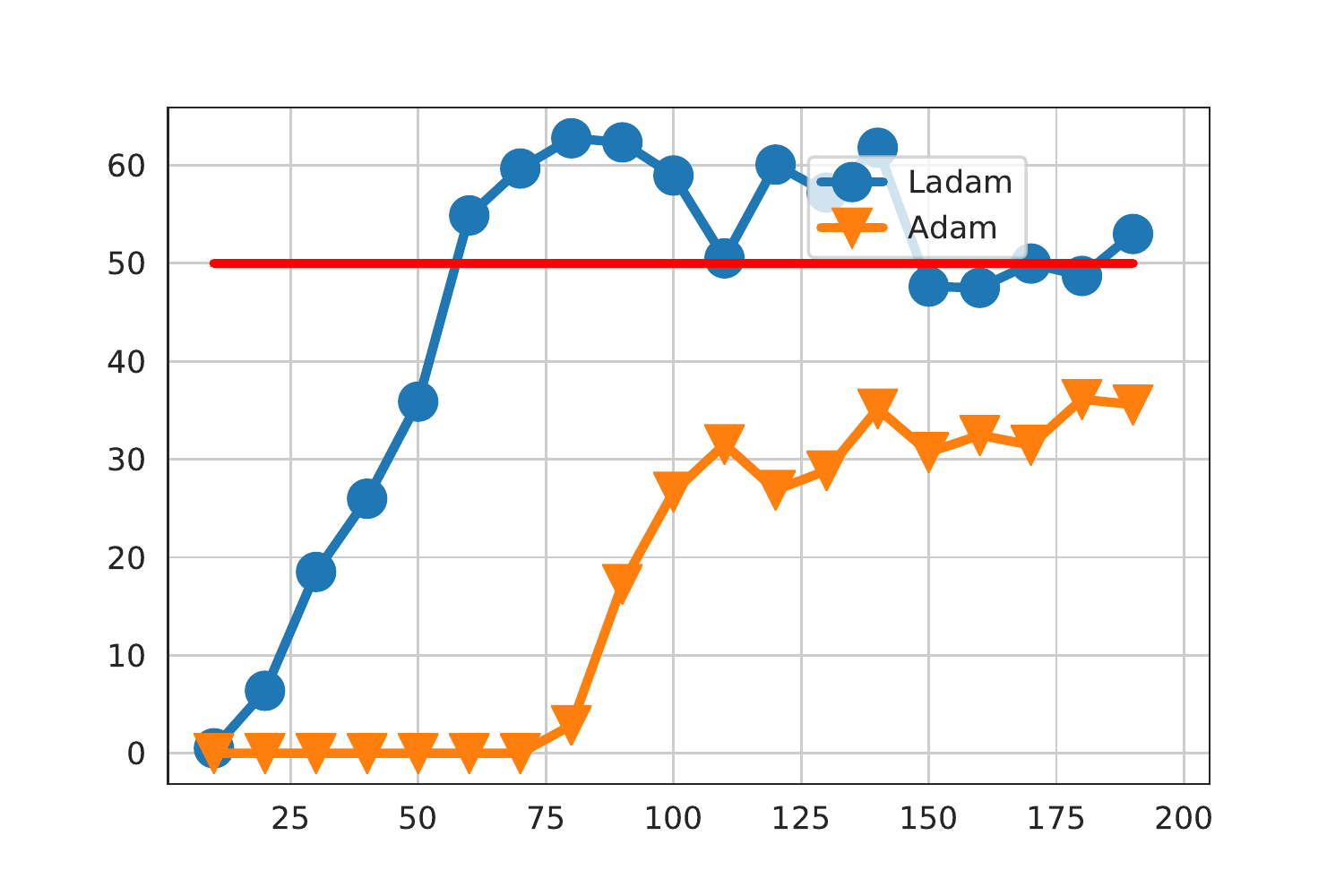}	\includegraphics[width=0.45\textwidth]{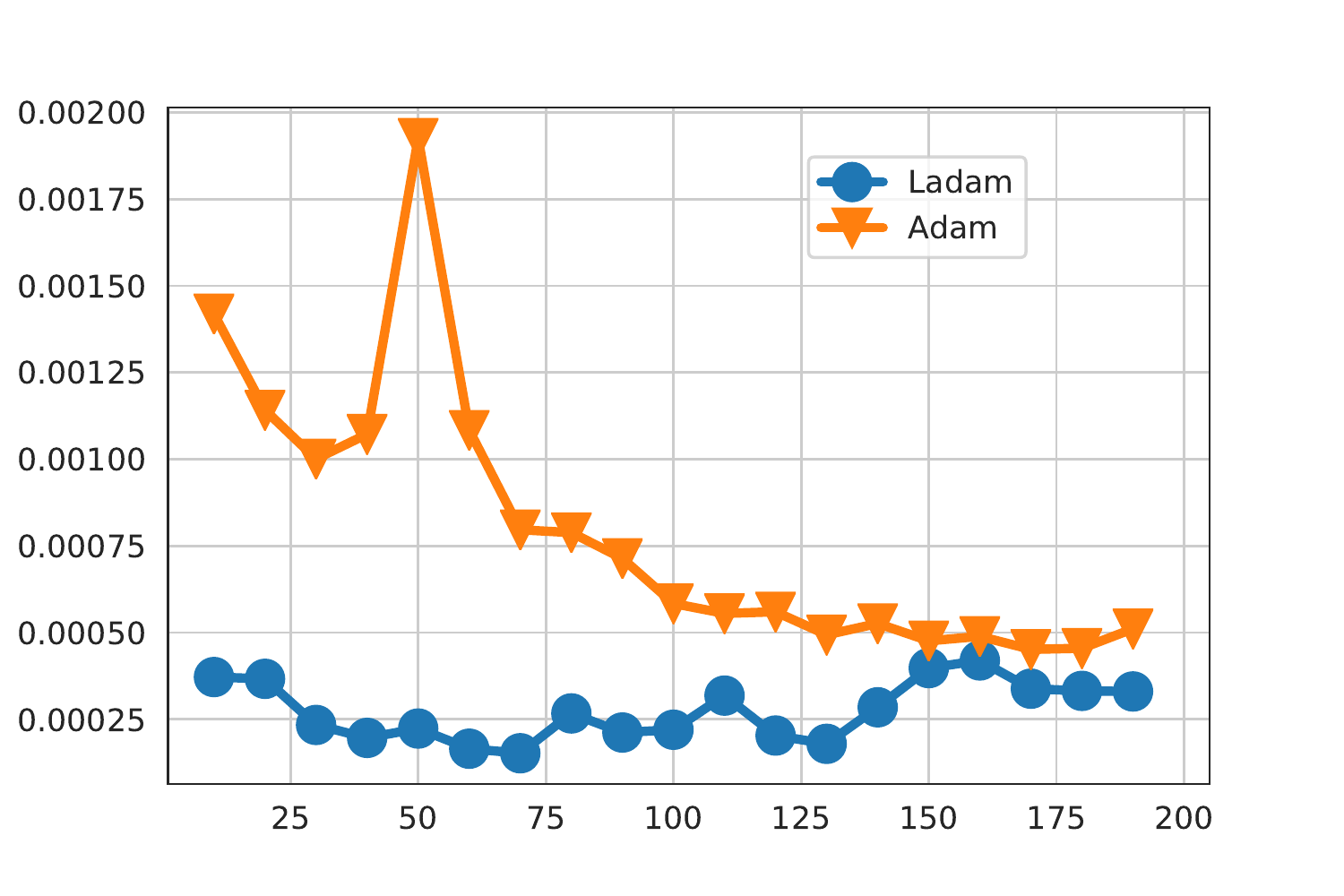}	
    \caption{Discriminator accuracy for ADAM and LADAM in (a), and generator loss for ADAM and LADAM in (b).}
    \label{Fig:Plot_loss_generator_acc_discrim_TimeGan}
\end{figure}

\newpage

\appendix

\section{Proof of Equation \eqref{Eq:sde23}}
\label{Appendix:proofOfSDE23}

To derive \eqref{Eq:sde23}, we follow the same approach of Section \ref{sec:LRAnalysis} and consider the value function below
\begin{equation}
\tilde{v}(t,q) = \min_{m^\theta \in \tilde{\mathcal{M}}^{\theta}} \Esp\big[g\big(\tilde{q}( T / m^\theta \big)\big|\tilde{q}(t) = q],
\label{Eq:ValFunctDef2}
\end{equation}
for any $(t,q) \in \mathbb{R}_+ \times \mathbb{R}^{M+N}$, with $\tilde{\mathcal{M}}^{\theta}$ the set of admissible controls for $m^{\theta}$, and 
\begin{equation}
\left\{
\begin{array}{l}
d\tilde{w}(t) =  g_{w}(\tilde{q}(t)) dt + \sqrt{\eta/(m^\theta)}  \sigma_w(\tilde{q}(t)) dW^1(t), \\
\\
d\tilde{\theta}(t) = -  g_{\theta}(\tilde{q}(t)) dt + \sqrt{\eta/(m^\theta)}  \sigma_\theta(\tilde{q}(t)) dW^2(t),
\end{array}
\right. \label{Eq:sde2F}
\end{equation}
where $\tilde{q}(t) = (\tilde{w}(t), \tilde{\theta}(t))$. The set $\tilde{\mathcal{M}}^{\theta}$ is defined as
\begin{align*}
\tilde{\mathcal{M}}^{\theta} = \big\{ m :\;& m \text{ c\`{a}dl\`{a}g in } [1,m^{\max}] \text{ adapted to } \mathbb{F}^{(W^1,W^2)},\\ 
					      &   \Esp[g\big( \tilde{q}( T / m \big) \,|\,\tilde{q}(0)\,] < \infty
					      \big\}.
\end{align*}
Since the batch size is a single parameter, one can either choose to maximize or minimize the objective function in \eqref{Eq:ValFunctDef2}. We choose here to reduce the loss as the ultimate goal of GANs is to generate realistic data.\\

For any $t \geq 0$ and $m^\theta \in [1,m^{\max}]$,  write $\tilde{q}^m = (\tilde{w}^m,\tilde{\theta}^m ) $ for the process $\tilde{q}^m(t) = \tilde{q}(m^\theta\, t)$. By It\^{o}'s formula and \eqref{Eq:sde2F}, we get
\begin{equation*}
\hspace{-0.3cm}
\left\{
\begin{array}{l}
d\tilde{w}^m(t) =   \cfrac{g_{w}(\tilde{q}^m(t))}{m^\theta} dt + \cfrac{\sqrt{\eta} \sigma_w(\tilde{q}^m(t))}{m^\theta } d\tilde{W}^1(t), \\
\\
d\tilde{\theta}^m(t) = -  \cfrac{ g_{\theta}(\tilde{q}^m(t))}{m^\theta } dt + \cfrac{\sqrt{\eta} \sigma_\theta(\tilde{q}^m(t))}{m^\theta } d\tilde{W}^2(t),
\end{array}
\right.
\end{equation*}
with $(\tilde{W}^1,\tilde{W}^2)(t) =  \sqrt{ m^\theta}(W^1,W^2)\big( t/ (m^{\theta})\big)$. The scaling property of the Brownian motion ensures that $(\tilde{W}^1,\tilde{W}^2)$ is a Brownian motion as well. Thus, we replace the value function in \eqref{Eq:ValFunctDef2} by the quantity below
\begin{equation*}
\tilde{v}^m(t,q) = \min_{m^\theta \in \mathcal{M}^{\theta}}  \Esp\big[g\big(\tilde{q}^m(T)\big)\big|\tilde{q}(t) = q],
\end{equation*}
to complete the proof.

\bibliographystyle{apalike}
\bibliography{gan_opti_lr_sde}

\begin{thebibliography}{}

\bibitem[Arjovsky and Bottou, 2017]{arjovsky2017towards}
Arjovsky, M. and Bottou, L. (2017).
\newblock Towards principled methods for training generative adversarial
  networks.
\newblock {\em arXiv preprint arXiv:1701.04862}.

\bibitem[Arjovsky et~al., 2017]{arjovsky2017wasserstein}
Arjovsky, M., Chintala, S., and Bottou, L. (2017).
\newblock Wasserstein generative adversarial networks.
\newblock In {\em International Conference on Machine Learning}, pages
  214--223. PMLR.

\bibitem[Barles and Souganidis, 1991]{barles1991convergence}
Barles, G. and Souganidis, P.~E. (1991).
\newblock Convergence of approximation schemes for fully nonlinear second order
  equations.
\newblock {\em Asymptotic Analysis}, 4(3):271--283.

\bibitem[Bayraktar and Yao, 2013]{bayraktar2013weak}
Bayraktar, E. and Yao, S. (2013).
\newblock {A weak dynamic programming principle for zero-sum stochastic
  differential games with unbounded controls}.
\newblock {\em SIAM Journal on Control and Optimization}, 51(3):2036--2080.

\bibitem[Berard et~al., 2020]{Berard2020}
Berard, H., Gidel, G., Almahairi, A., Vincent, P., and Lacoste-Julien, S.
  (2020).
\newblock {A closer look at the optimization landscape of generative
  adversarial networks}.
\newblock In {\em International Conference on Learning Representations}.

\bibitem[Bouchard and Touzi, 2011]{bouchard2011weak}
Bouchard, B. and Touzi, N. (2011).
\newblock Weak dynamic programming principle for viscosity solutions.
\newblock {\em SIAM Journal on Control and Optimization}, 49(3):948--962.

\bibitem[Cao and Guo, 2020]{cao2020approximation}
Cao, H. and Guo, X. (2020).
\newblock {Approximation and convergence of {GAN}s training: an {SDE}
  approach}.
\newblock {\em arXiv preprint arXiv:2006.02047}.

\bibitem[Cao and Guo, 2021]{cao2021generative}
Cao, H. and Guo, X. (2021).
\newblock Generative dversarial network: some analytical perspectives.
\newblock {\em Machine Learning And Data Sciences For Financial Markets: A
  Guide To Contemporary Practices}.

\bibitem[Chakraborti et~al., 2011]{chakraborti2011econophysics}
Chakraborti, A., Toke, I.~M., Patriarca, M., and Abergel, F. (2011).
\newblock {Econophysics review: I. empirical facts}.
\newblock {\em Quantitative Finance}, 11(7):991--1012.

\bibitem[Conforti et~al., 2020]{Conforti2020}
Conforti, G., Kazeykina, A., and Ren, Z. (2020).
\newblock {Game on random environment, mean-field Langevin system and neural
  networks}.
\newblock {\em arXiv preprint arXiv:2004.02457}.

\bibitem[Cont, 2001]{cont2001empirical}
Cont, R. (2001).
\newblock Empirical properties of asset returns: stylized facts and statistical
  issues.
\newblock {\em Quantitative finance}, 1(2):223.

\bibitem[Denton et~al., 2015]{denton2015deep}
Denton, E.~L., Chintala, S., Szlam, A., and Fergus, R. (2015).
\newblock Deep generative image models using a {Laplacian} pyramid of
  adversarial networks.
\newblock In {\em Advances in Neural Information Processing Systems}, pages
  1486--1494.

\bibitem[Dionelis et~al., 2020]{dionelis2020tail}
Dionelis, N., Yaghoobi, M., and Tsaftaris, S.~A. (2020).
\newblock Tail of distribution {GAN} ({TailGAN}):
  Generativeadversarial-network-based boundary formation.
\newblock In {\em 2020 Sensor Signal Processing for Defence Conference (SSPD)},
  pages 1--5. IEEE.

\bibitem[Domingo-Enrich et~al., 2020]{Domingo-Enrich2020}
Domingo-Enrich, C., Jelassi, S., Mensch, A., Rotskoff, G.~M., and Bruna, J.
  (2020).
\newblock {A mean-field analysis of two-player zero-sum games}.
\newblock {\em arXiv preprint arXiv:2002.06277}.

\bibitem[Eckerli and Osterrieder, 2021]{eckerli2021generative}
Eckerli, F. and Osterrieder, J. (2021).
\newblock Generative {a}dversarial {n}etworks in finance: an overview.
\newblock {\em Available at SSRN 3864965}.

\bibitem[Efimov et~al., 2020]{efimov2020using}
Efimov, D., Xu, D., Kong, L., Nefedov, A., and Anandakrishnan, A. (2020).
\newblock Using generative adversarial networks to synthesize artificial
  financial datasets.
\newblock {\em arXiv preprint arXiv:2002.02271}.

\bibitem[Evans, 1983]{evans1983classical}
Evans, L.~C. (1983).
\newblock {Classical solutions of the Hamilton-Jacobi-Bellman equation for
  uniformly elliptic operators}.
\newblock {\em Transactions of the American Mathematical Society},
  275(1):245--255.

\bibitem[Evans and Souganidis, 1984]{evans1984differential}
Evans, L.~C. and Souganidis, P.~E. (1984).
\newblock {Differential games and representation formulas for solutions of
  Hamilton-Jacobi-Isaacs equations}.
\newblock {\em Indiana University Mathematics Journal}, 33(5):773--797.

\bibitem[Fatkullin and Vanden-Eijnden, 2004]{fatkullin2004computational}
Fatkullin, I. and Vanden-Eijnden, E. (2004).
\newblock {A computational strategy for multiscale systems with applications to
  Lorenz 96 model}.
\newblock {\em Journal of Computational Physics}, 200(2):605--638.

\bibitem[Fu et~al., 2019]{fu2019time}
Fu, R., Chen, J., Zeng, S., Zhuang, Y., and Sudjianto, A. (2019).
\newblock Time series simulation by conditional generative adversarial net.
\newblock {\em arXiv preprint arXiv:1904.11419}.

\bibitem[Gadat and Panloup, 2017]{gadat2017optimal}
Gadat, S. and Panloup, F. (2017).
\newblock {Optimal non-asymptotic bound of the Ruppert-Polyak averaging without
  strong convexity}.
\newblock {\em arXiv preprint arXiv:1709.03342}.

\bibitem[Goodfellow et~al., 2014]{goodfellow2014generative}
Goodfellow, I., Pouget-Abadie, J., Mirza, M., Xu, B., Warde-Farley, D., Ozair,
  S., Courville, A., and Bengio, Y. (2014).
\newblock Generative adversarial nets.
\newblock In {\em {Advances in Neural Information Processing Systems}}, pages
  2672--2680.

\bibitem[Gulrajani et~al., 2017]{gulrajani2017improved}
Gulrajani, I., Ahmed, F., Arjovsky, M., Dumoulin, V., and Courville, A. (2017).
\newblock {Improved training of {W}asserstein {GAN}s}.
\newblock {\em arXiv preprint arXiv:1704.00028}.

\bibitem[Hsieh et~al., 2019]{hsieh2019finding}
Hsieh, Y.-P., Liu, C., and Cevher, V. (2019).
\newblock Finding mixed {N}ash equilibria of generative adversarial networks.
\newblock In Chaudhuri, K. and Salakhutdinov, R., editors, {\em Proceedings of
  the 36th International Conference on Machine Learning}, volume~97 of {\em
  Proceedings of Machine Learning Research}, pages 2810--2819. PMLR.

\bibitem[Jacod and Shiryaev, 2013]{jacod2013limit}
Jacod, J. and Shiryaev, A. (2013).
\newblock {\em Limit {T}heorems for {S}tochastic {P}rocesses}, volume 288.
\newblock Springer Science \& Business Media.

\bibitem[Kamalaruban et~al., 2020]{kamalaruban2020robust}
Kamalaruban, P., Huang, Y.-T., Hsieh, Y.-P., Rolland, P., Shi, C., and Cevher,
  V. (2020).
\newblock {Robust reinforcement learning via adversarial training with langevin
  dynamics}.
\newblock {\em arXiv preprint arXiv:2002.06063}.

\bibitem[Krylov, 2014]{krylov2014dynamic}
Krylov, N. (2014).
\newblock {On the dynamic programming principle for uniformly nondegenerate
  stochastic differential games in domains and the Isaacs equations}.
\newblock {\em Probability Theory and Related Fields}, 158(3-4):751--783.

\bibitem[Kulharia et~al., 2017]{ghosh2016contextual}
Kulharia, V., Ghosh, A., Mukerjee, A., Namboodiri, V., and Bansal, M. (2017).
\newblock {Contextual RNN-GANs for abstract reasoning diagram generation}.
\newblock {\em Proceedings of the AAAI Conference on Artificial Intelligence},
  31(1).

\bibitem[Kushner et~al., 2001]{kushner2001numerical}
Kushner, H. J.~K., Kushner, H.~J., Dupuis, P.~G., and Dupuis, P. (2001).
\newblock {\em {Numerical Methods for Stochastic Control Problems in Continuous
  Time}}, volume~24.
\newblock Springer Science \& Business Media.

\bibitem[Ledig et~al., 2017]{ledig2016others}
Ledig, C., Theis, L., Husz{\'a}r, F., Caballero, J., Cunningham, A., Acosta,
  A., Aitken, A., Tejani, A., Totz, J., Wang, Z., et~al. (2017).
\newblock Photo-realistic single image super-resolution using a generative
  adversarial network.
\newblock In {\em Proceedings of the IEEE Conference on Computer Vision and
  Pattern Recognition}, pages 4681--4690.

\bibitem[Li et~al., 2020]{li2020generating}
Li, J., Wang, X., Lin, Y., Sinha, A., and Wellman, M. (2020).
\newblock {Generating realistic stock market order streams}.
\newblock {\em Proceedings of the AAAI Conference on Artificial Intelligence},
  34(01):727--734.

\bibitem[Luc et~al., 2016]{luc2016semantic}
Luc, P., Couprie, C., Chintala, S., and Verbeek, J. (2016).
\newblock Semantic segmentation using adversarial networks.
\newblock {\em arXiv preprint arXiv:1611.08408}.

\bibitem[Marti, 2020]{marti2020corrgan}
Marti, G. (2020).
\newblock Corr{G}an: sampling realistic financial correlation matrices using
  generative adversarial networks.
\newblock In {\em ICASSP 2020-2020 IEEE International Conference on Acoustics,
  Speech and Signal Processing (ICASSP)}, pages 8459--8463. IEEE.

\bibitem[Mescheder et~al., 2018]{Mescheder2018}
Mescheder, L., Geiger, A., and Nowozin, S. (2018).
\newblock Which training methods for {GAN}s do actually converge?
\newblock In Dy, J. and Krause, A., editors, {\em Proceedings of the 35th
  International Conference on Machine Learning}, volume~80 of {\em Proceedings
  of Machine Learning Research}, pages 3481--3490. PMLR.

\bibitem[Mirza and Osindero, 2014]{mirza2014conditional}
Mirza, M. and Osindero, S. (2014).
\newblock Conditional generative adversarial nets.
\newblock {\em arXiv preprint arXiv:1411.1784}.

\bibitem[Moulines and Bach, 2011]{moulines2011non}
Moulines, E. and Bach, F.~R. (2011).
\newblock Non-asymptotic analysis of stochastic approximation algorithms for
  machine learning.
\newblock In {\em Advances in Neural Information Processing Systems}, pages
  451--459.

\bibitem[Mounjid and Lehalle, 2019]{mounjid2019improving}
Mounjid, O. and Lehalle, C.-A. (2019).
\newblock Improving reinforcement learning algorithms: towards optimal learning
  rate policies.
\newblock {\em arXiv preprint arXiv:1911.02319}.

\bibitem[Ni et~al., 2020]{ni2020conditional}
Ni, H., Szpruch, L., Wiese, M., Liao, S., and Xiao, B. (2020).
\newblock Conditional {S}ig-{W}asserstein {GAN}s for time series generation.
\newblock {\em arXiv preprint arXiv:2006.05421}.

\bibitem[Pham, 2009]{pham2009continuous}
Pham, H. (2009).
\newblock {\em {Continuous-time Stochastic Control and Optimization with
  Financial Applications}}, volume~61.
\newblock Springer Science \& Business Media.

\bibitem[Pimentel, 2019]{pimentel2019regularity}
Pimentel, E.~A. (2019).
\newblock {Regularity theory for the Isaacs equation through approximation
  methods}.
\newblock In {\em Annales de l'Institut Henri Poincar{\'e} C, Analyse non
  {L}in{\'e}aire}, volume~36, pages 53--74. Elsevier.

\bibitem[Radford et~al., 2015]{radford2015unsupervised}
Radford, A., Metz, L., and Chintala, S. (2015).
\newblock Unsupervised representation learning with deep convolutional
  generative adversarial networks.
\newblock {\em arXiv preprint arXiv:1511.06434}.

\bibitem[Reed et~al., 2016]{reed2016generative}
Reed, S., Akata, Z., Yan, X., Logeswaran, L., Schiele, B., and Lee, H. (2016).
\newblock {Generative Adversarial Text to Image Synthesis}.
\newblock In {\em 33rd International Conference on Machine Learning}, pages
  1060--1069.

\bibitem[Salimans et~al., 2016]{Salimans2016}
Salimans, T., Goodfellow, I., Zaremba, W., Cheung, V., Radford, A., and Chen,
  X. (2016).
\newblock {Improved techniques for training GANs}.
\newblock In {\em Advances in Neural Information Processing Systems}, pages
  2234--2242.

\bibitem[Sion, 1958]{sion1958general}
Sion, M. (1958).
\newblock On general minimax theorems.
\newblock {\em {Pacific Journal of Mathematics}}, 8(1):171--176.

\bibitem[Sirbu, 2014a]{sirbu2014martingale}
Sirbu, M. (2014a).
\newblock {On martingale problems with continuous-time mixing and values of
  zero-sum games without the Isaacs condition}.
\newblock {\em SIAM Journal on Control and Optimization}, 52(5):2877--2890.

\bibitem[Sirbu, 2014b]{sirbu2014stochastic}
Sirbu, M. (2014b).
\newblock {Stochastic Perron's method and elementary strategies for zero-sum
  differential games}.
\newblock {\em SIAM Journal on Control and Optimization}, 52(3):1693--1711.

\bibitem[Storchan et~al., 2020]{storchan2020mas}
Storchan, V., Vyetrenko, S., and Balch, T. (2020).
\newblock {M{AS-GAN}: adversarial calibration of multi-agent market
  simulators.}

\bibitem[Takahashi et~al., 2019]{takahashi2019modeling}
Takahashi, S., Chen, Y., and Tanaka-Ishii, K. (2019).
\newblock Modeling financial time-series with generative adversarial networks.
\newblock {\em Physica A: Statistical Mechanics and its Applications},
  527:121261.

\bibitem[Von~Neumann, 1959]{von1959theory}
Von~Neumann, J. (1959).
\newblock On the theory of games of strategy.
\newblock {\em Contributions to the Theory of Games}, 4:13--42.

\bibitem[Vondrick et~al., 2016]{vondrick2016generating}
Vondrick, C., Pirsiavash, H., and Torralba, A. (2016).
\newblock Generating videos with scene dynamics.
\newblock In {\em Advances in Neural Information Processing Systems}, pages
  613--621.

\bibitem[Wang and Forsyth, 2008]{wang2008maximal}
Wang, J. and Forsyth, P.~A. (2008).
\newblock {Maximal use of central differencing for Hamilton--Jacobi--Bellman
  PDEs in finance}.
\newblock {\em SIAM Journal on Numerical Analysis}, 46(3):1580--1601.

\bibitem[Weinan et~al., 2005]{weinan2005analysis}
Weinan, E., Liu, D., and Vanden-Eijnden, E. (2005).
\newblock Analysis of multiscale methods for stochastic differential equations.
\newblock {\em Communications on Pure and Applied Mathematics},
  58(11):1544--1585.

\bibitem[Wiese et~al., 2019]{Wiese2019}
Wiese, M., Bai, L., Wood, B., Morgan, J.~P., and Buehler, H. (2019).
\newblock {Deep hedging: learning to simulate equity option markets}.
\newblock {\em arXiv preprint arXiv:1911.01700}.

\bibitem[Wiese et~al., 2020]{wiese2020quant}
Wiese, M., Knobloch, R., Korn, R., and Kretschmer, P. (2020).
\newblock Quant {GAN}s: deep generation of financial time series.
\newblock {\em Quantitative Finance}, 20(9):1419--1440.

\bibitem[Yeh et~al., 2016]{yeh2016semantic}
Yeh, R., Chen, C., Lim, T.~Y., Hasegawa-Johnson, M., and Do, M.~N. (2016).
\newblock Semantic image inpainting with perceptual and contextual losses.
\newblock {\em arXiv preprint arXiv:1607.07539}, 2(3).

\bibitem[Yoon et~al., 2019]{yoon2019time}
Yoon, J., Jarrett, D., and Van~der Schaar, M. (2019).
\newblock Time-series generative adversarial networks.
\newblock {\em Neural Information Processing Systems}.

\bibitem[Zhu et~al., 2016]{zhu2016generative}
Zhu, J.-Y., Kr{\"a}henb{\"u}hl, P., Shechtman, E., and Efros, A.~A. (2016).
\newblock Generative visual manipulation on the natural image manifold.
\newblock In {\em European Conference on Computer Vision}, pages 597--613.
  Springer.

\bibitem[Zumbach and Lynch, 2001]{zumbach2001heterogeneous}
Zumbach, G. and Lynch, P. (2001).
\newblock Heterogeneous volatility cascade in financial markets.
\newblock {\em Physica A: Statistical Mechanics and its Applications},
  298(3-4):521--529.

\end{thebibliography}

\end{document}